\documentclass[citeauthoryear]{llncs}
\usepackage{llncsdoc}
\usepackage{times}
\usepackage{epsfig}
\usepackage{graphicx}
\usepackage{amsmath}
\usepackage{amsfonts}
\usepackage{algorithm}
\usepackage{algorithmic}
\usepackage{multirow}

\begin{document}

\title{Manifold Elastic Net: A Unified Framework for\\Sparse Dimension Reduction}
\author{Tianyi Zhou\inst{1} \and Dacheng Tao\inst{1} \and Xindong Wu\inst{2}}
\institute{School of Computer Engineering, Nanyang Technological University, Singapore 639798. \email{dctao@ntu.edu.sg} \and Department of Computer Science, University of Vermont, 33 Colchester Avenue, Burlington, Vermont 05405, USA.}
\maketitle

\begin{abstract}
It is difficult to find the optimal sparse solution of a manifold learning based dimensionality reduction algorithm. The lasso or the elastic net penalized manifold learning based dimensionality reduction is not directly a lasso penalized least square problem and thus the least angle regression (LARS) (Efron et al. \cite{LARS}), one of the most popular algorithms in sparse learning, cannot be applied. Therefore, most current approaches take indirect ways or have strict settings, which can be inconvenient for applications. In this paper, we proposed the manifold elastic net or MEN for short. MEN incorporates the merits of both the manifold learning based dimensionality reduction and the sparse learning based dimensionality reduction. By using a series of equivalent transformations, we show MEN is equivalent to the lasso penalized least square problem and thus LARS is adopted to obtain the optimal sparse solution of MEN. In particular, MEN has the following advantages for subsequent classification: 1) the local geometry of samples is well preserved for low dimensional data representation, 2) both the margin maximization and the classification error minimization are considered for sparse projection calculation, 3) the projection matrix of MEN improves the parsimony in computation, 4) the elastic net penalty reduces the over-fitting problem, and 5) the projection matrix of MEN can be interpreted psychologically and physiologically. Experimental evidence on face recognition over various popular datasets suggests that MEN is superior to top level dimensionality reduction algorithms.
\end{abstract}

\section{Introduction}

One of the primary focuses in data mining and machine learning is finding a succinct and effective representation for original high dimensional samples (Hastie et al. \cite{HastieStatistics}; Kriegel et al. \cite{DataMining}; Ding and Li \cite{Ding1}; Ding et al. \cite{NMF}; Li et al. \cite{TextCate}; Tao et al. \cite{SupervisedTensorLearning}; Tao et al. \cite{TDA}). Linear dimensionality deduction is such a tool that projects the original samples from a high dimensional space to a low dimensional subspace. Meanwhile some particular information, e.g., manifold structure and discriminative information, of the original high dimensional samples will be well preserved while noises will be removed in the selected subspace.

\subsection{The state of the art}

In the past decades, a dozen of algorithms have been developed and extensive experimental results have demonstrated that duly selected subspace is effective and efficient for subsequent utilizations. In this paper, we categorize popular dimensionality reduction algorithms into the following three groups:

\begin{enumerate}

\item Conventional linear dimensionality reduction algorithms, e.g., principal components analysis (PCA) (Hotelling \cite{PCA}), Fisher's linear discriminant analysis (FLDA) (Fisher \cite{FisherLDA}), regularized FLDA, and the geometric mean based subspace selection (Tao et al. \cite{GeometricMean}). All of these algorithms assume samples are drawn from different Gaussians. PCA maximizes the mutual information between original high-dimensional Gaussian distributed samples and projected low-dimensional samples. PCA, which is unsupervised, does not utilize the class label information. While, LDA finds a projection matrix that maximizes the trace of the between-class scatter matrix and minimizes the trace of the within-class scatter matrix in the projected subspace simultaneously. The same as PCA, FLDA and regularized FLDA assume samples are drawn from homoscedastic Gaussians. Therefore, FLDA and regularized FLDA cannot work well when Gaussians are heteroscedastic. Additionally, they always merge classes which are close in the high dimensional space. Although the geometric mean based subspace selection and its harmonic mean based extension (Bian and Tao \cite{HarmonicMean}) assume samples are drawn from heteroscedastic Gaussians and do not tend to merge close classes, they basically work for Gaussian distributed samples.

\item Manifold learning based dimensionality reduction algorithms: e.g., locally linear embedding (LLE) (Roweis and Saul \cite{LLE}), ISOMAP (Tenenbaum et al. \cite{ISOMAP}), Laplacian eigenmaps (LE) (Belkin and Niyogi \cite{LaplacianEigenmap}; Li et al. \cite{DLLETensor}), Hessian eigenmaps (HLLE) (Donoho and Grimes \cite{HessianEigenmap}), Generative Topographic Mapping (GTM) (Bishop et al. \cite{GTM}; Fyfe \cite{TopographicMaps}) and local tangent space alignment (LTSA) (Zhang and Zha \cite{LSTA}). LLE uses linear coefficients, which reconstruct a given measurement by its neighbours, to represent the local geometry, and then seeks a low-dimensional embedding, in which these coefficients are still suitable for reconstruction. ISOMAP preserves global geodesic distances of all pairs of measurements. LE preserves proximity relationships by manipulations on an undirected weighted graph, which indicates neighbour relations of pairwise measurements. LTSA exploits the local tangent information as a representation of the local geometry and this local tangent information is then aligned to provide a global coordinate. HLLE obtains the final low-dimensional representations by applying eigen-analysis to a matrix which is built by estimating the Hessian over neighbourhood. All these algorithms have the out of sample problem and thus a dozen of linearizations have been proposed, e.g., locality preserving projections (LPP) (He and Niyogi \cite{LPP}), neighborhood preserving embedding (NPE) (He et al. \cite{NPE}), and orthogonal neighbourhood preserving projections (ONPP). Recently, we provide a systematic framework, i.e., patch alignment (Zhang et al. \cite{DLAeccv}; Zhang et al. \cite{DLAtkde}), for understanding the common properties and intrinsic difference in different algorithms including their linearizations. In particular, this framework reveals that: i) algorithms are intrinsically different in the patch optimization stage; and ii) all algorithms share an almost-identical whole alignment stage. Another unified view of popular manifold learning algorithms is the graph embedding framework (Yan et al. \cite{GE}). Based on both frameworks, different algorithms have been developed, e.g., the discriminative locality alignment (Liu et al. \cite{TCA}), manifold regularization (Belkin et al. \cite{ManifoldRegularization}) and marginal Fisher's analysis (Wang et al. \cite{FeiWang1}).

\item Sparse learning based dimensionality reduction algorithms: e.g., lasso (Tibshirani \cite{lasso}), elastic net (Zou and Hastie \cite{ElasticNet}), the smoothly clipped absolute deviation penalty (SCAD) (Fan and Li \cite{SCAD}), Sure independence screening (Fan and Lv \cite{SIC}), Dantzig selector (Candes and Tao \cite{DantzigSelector}) and Dantzig selector with sequential optimization (Dasso) (James et al. \cite{Dasso}). Conventional linear dimensionality reduction algorithms and manifold learning based dimensionality reduction algorithms produce a low dimensional subspace and each basis of the subspace is a linear combination of all the original bases (i.e., variables or features) used for high dimensional sample representation. Therefore, results cannot be interpreted psychologically and physiologically. Sparse learning based dimensionality reduction algorithms are developed not only to achieve the dimensionality reduction but also to reduce the number of explicitly used variables. A direct method to reduce the number of variables for representation is setting very small coefficients as zero. However, this strategy is problematic because small coefficients could be very important. Because each of new bases is a linear combination of original ones, it is reasonable to consider each new basis as the response of several variables, i.e., the original features. Then the problem of sparse learning becomes a similar problem to variables selection and coefficients shrinkage. In linear regression, Lp norm penalty is always combined with the loss function to reduce over-fitting. In particular, $\ell_1$-norm (or lasso) owns a good property to drive a good number of coefficients to zero and lead to a sparse model between responses and variables because of its singularity in the origin (Park and Hastie \cite{L1RegulizationPath}; Huang and Chris Ding \cite{TensorFactorization}). The number of lasso selected variables is no larger than the number of samples. Moreover, lasso randomly selects one from the group of variables that are high correlated. Therefore, elastic net is proposed to address the above problems and achieve the grouping effect by adding the $\ell_2$ penalty to lasso.

\end{enumerate}

In recent years, sparse learning becomes popular, because:

\begin{enumerate}

\item sparsity can make the data more succinct and simpler, so the calculation of the low dimensional representation and the subsequent processing, e.g., classification and regression, becomes more efficient. Parsimony is especially an important factor when the dimension of the original samples is very high and the number of samples is very large;

\item sparsity can control the weights of original variables and decrease the variance brought by possible over-fitting with the least increment of the bias. Therefore, the learn model can generalize better; and

\item sparsity provides a good interpretation of a model, thus reveals an explicit relationship between the objective of the model and the given variables. This is important for understanding practical problems, especially when the number of variables is larger than that of the samples.

\end{enumerate}

However, it is not easy to find the optimal solution of a sparse learning model. In the original lasso, the residue sum of squares is minimized subject to the sum of the absolute value of the coefficients being less than a constant. The quadratic programming is sequentially utilized to get the solution and thus the time cost is not acceptable for practical applications. Recently, the least angle regression (LARS) is proposed to seek a close form solution to the path of coefficients in each step without using the quadratic programming, so it is more efficient and less greedy than the original optimization algorithm used in lasso.

Hitherto, most of sparse dimensionality reduction algorithms are designed for linear regression and only a few can be applied for subsequent classification, e.g., sparse principal component analysis (SPCA) (Zou and Hastie \cite{SparsePCAZou}), Nonnegative sparse principal component analysis (Zass and Shashua \cite{NonSpaPCA}), sparse linear discriminant analysis (SLDA), sparse projections over graph (SPOG) (Cai et al. \cite{SpectralRegression}; Cai et al. \cite{SRDA}) and SPCA using semi-definite programming (Aspremont et al. \cite{SparsePCAAlex}). Both SPCA and SPCA using semi-definite programming do not consider the sample label information and thus some discriminative information will be removed after dimensionality reduction. SLDA can work well for binary class classification but it cannot be applied for multi-class classification. SPOG utilizes a particular manifold learning based dimensionality reduction algorithm, e.g., locality preserving projections (LPP), to obtain the dense projection matrix and then applies lasso to regress the corresponding sparse projection matrix. Absolutely the problem is indirectly formulated to obtain the sparse projection matrix. A direct formulation should be imposing the lasso penalty over a loss function (i.e., a criterion) of a dimensionality reduction algorithm. However, it is difficult to use LARS to obtain its optimal solution because the objective function is not a direct regression problem. Therefore, researchers currently take indirect routs to obtain sparse projection matrices.

\subsection{The proposed approach}

In this paper, we propose the manifold elastic net (MEN), which obtains a sparse projection matrix for subsequent classification. MEN directly imposes the elastic net penalty (i.e., the combination of the lasso penalty and the $\ell_2$-norm penalty) over the loss (i.e., the criterion) of a discriminative manifold learning based dimensionality reduction algorithm. By using a series of complex linear algebra equivalent transformations, the objective function of MEN can be rewritten as a lasso penalized least square problem and thus LARS can be applied to obtain the optimal sparse solution of MEN.

In detail, we first apply the part optimization of the patch alignment framework to encode the local geometry of a set of training samples. In the second step, the whole alignment of the patch alignment framework is applied to calculate the unified coordinate system for local patches obtained in the first step. For low dimensional data representation, the linearization or the linear approximation is adopted in MEN. Although we can impose some discriminative information preservation criterion (e.g., margin maximization) over the part optimization stage, it is not directly relevant to the classification error minimization. Therefore, we put a new item that minimizes the classification error in the third step. To obtain a sparse projection matrix with the grouping effect, in the fourth step, the elastic net penalty is adopted in MEN. So far, the objective function of MEN is fully constructed.

With the well defined MEN, we then apply LARS to obtain the optimal solution of MEN. We transform MEN into a form in which the correlation of basis can be written as the correlation of coefficients. Active set is built according to LARS. In each step, no more than one element of the basis is added to the active set according to its correlation. All elements in the active set are changed in each step with special direction and distance in the space of coefficients. The direction and distance of a path in each step have closed form solution according to the extended simplex. The sparsity of the projection matrix is controlled by the cardinality of the active set. Because the LARS for MEN generates bases in an independent way, the same procedure is conducted multiple times to obtain a set of bases. Under this procedure, these bases are orthogonal. Thorough experiments on face recognition (Shakhnarovich and Moghaddam \cite{FaceRecognitionSubspaces}) task based on popular face datasets show the effectiveness of the proposed MEN by comparing against the top level dimensionality reduction algorithms.

The rest of the paper is organized as follows. Section 2 presents the proposed manifold elastic net (MEN) including the objective function of MEN and the LARS optimization for MEN. Section 3 shows the effectiveness of MEN for face recognition over different face datasets. Section 4 concludes.

\section{Manifold Elastic Net}

Consider in the discriminative dimensionality reduction problem with   training samples and   corresponding class labels. Let $X=\left[x_1,x_2,\cdots,x_n\right]^T\in \mathbb R^{n\times p}$ be a given training set in a high dimensional space $\mathbb R^{n\times p}$ and $C=\left[c_1,c_2,\cdots,c_n\right]^T\in \mathbb R^n$ be the corresponding class label vector. The objective here is to find a projection matrix $W=\left[w_1,w_2,\cdots,w_d\right]^T\in \mathbb R^{p\times d}$ that projects samples $x^T\in \mathbb R^p$ in the high dimensional space onto a low dimensional subspace, i.e., $z^T=x^TW$, such that samples from different classes can be well separate, i.e., the classification error can be extremely minimized.

Manifold learning based dimensionality reduction aims to find the corresponding low dimensional representation z in a low dimensional Euclidean space of x to preserve (actually approximate) the data intrinsic structure. Popular manifold learning based dimensionality reduction algorithms, however, have the following two problems: 1) the classification error is not directly and explicitly considered, although some algorithms compound discriminative information preservation criteria, e.g., margin maximization; and 2) the obtained low dimensional representation linear combines of all variables in the high dimensional space, so it is difficult to clear interpret and efficiently represent data.

Sparse learning provides sparse data representation via variable selection, and has the following advantages: 1) the sparsity improves the parsimony in computation, i.e., the computational cost can be significantly reduce; 2) the penalties and the constraints introduced in a learning model discourage the possible over-fitting of the model; and 3) the learned model can be well interpreted. However, existing sparse learning algorithms are designed for linear regression problems and the data intrinsic structure is usually ignored.

To achieve the merits of manifold learning based dimensionality reduction and the advantages of sparse learning, in this paper, we propose the manifold elastic net (MEN), which is a general framework to obtain the sparse solution of the manifold learning based discriminative dimensionality reduction. There are few research results on combining sparse learning and discriminative dimensionality reduction because the projection matrix of a lasso penalized model cannot be obtained directly by using the least angle regression (LARS).

MEN is not a direct combination of the manifold learning based dimensionality reduction and the sparse learning. It however finds the optimal sparse solution of every manifold learning based discriminative dimensionality reduction algorithm via the patch alignment framework and a new classification error minimization based criterion. In particular, MEN encodes the local geometry of a set of samples and finds an aligned coordinate system for data representation under the patch alignment framework; MEN utilizes the classification error minimization criterion to directly link the classification error with the selected subspace; and MEN incorporates the elastic net regularization to sparsify the projection matrix.

\subsection{Part optimization}

Different manifold learning algorithms encode different types of local geometry of samples, e.g., locally linear embedding (LLE) applies linear coefficients to reconstruct a sample by its neighbors. The patch alignment framework has well demonstrated that different algorithms have different optimization criteria to encode different local geometry over patches.

In MEN, the same as the part optimization in the patch alignment framework, each patch is constructed by a particular sample $x_i$ and its $k$ related ones $x_{i_1},x_{i_1},\cdots,x_{i_k}$. The patch is denoted by $X_i=\left[x_i^T,x_{i_1}^T,x_{i_2}^T,\cdots,x_{i_k}^T\right]^T\in \mathbb R^{(k+1)\times p}$. MEN finds a linear mapping $f_i$ that projects the patch $X_i\in \mathbb R^p$ to a low dimensional subspace $\mathbb R^d$, i.e., $f_i:X_i\longmapsto Z_i$, where $Z_i=\left[z_i^T,z_{i_1}^T,z_{i_2}^T,\cdots,z_{i_k}^T\right]^T\in \mathbb R^{(k+1)\times d}$. The part optimization maximizes the similarity of the local geometry represented by $X_i$ and that described by $Z_i$:
\begin{align}
\arg\min_{Z_i} {\rm tr}\left(Z_i^TL_iZ_i\right),
\label{equ:partopt}
\end{align}
where $L_i\in \mathbb R^{\left(k+1\right)\times \left(k+1\right)}$ encodes the local geometry of the patch $X_i$ and it is different over different dimensionality reduction algorithms.

For a given sample $x_i$, its $k$ related ones are divided into two groups: the $k_1$ ones in the same class with $x_i$ and the $k_2$ ones from different classes with $x_i$. These two groups are selected independently and denoted by $\left\{x_{i^1},x_{i^2},\cdots,x_{i^{k_1}}\right\}$ and $\left\{x_{i_1},x_{i_2},\cdots,x_{i_{k_1}}\right\}$ respectively. Therefore, the patch for $x_i$ is defined by
\begin{align}
\notag X_i=\left[x_i^T,x_{i^1}^T,x_{i^2}^T,\cdots,x_{i^{k_1}}^T,x_{i_1}^T,x_{i_2}^T,\cdots,x_{i_{k_1}}^T\right]^T\in \mathbb R^{\left(k_1+k_2+1\right)\times p}.
\end{align}
The corresponding the low dimensional representation is
\begin{align}
\notag Z_i=\left[z_i^T,z_{i^1}^T,z_{i^2}^T,\cdots,z_{i^{k_1}}^T,z_{i_1}^T,z_{i_2}^T,\cdots,z_{i_{k_1}}^T\right]^T\in \mathbb R^{\left(k_1+k_2+1\right)\times d}.
\end{align}
Let $F_i=\left\{i,i^1,i^2,\cdots,i^{k_1},i_1,i_2,\cdots,i_{k_2}\right\}$ to be the index set. In the low dimensional subspace, we expect that the distances between the given sample and the group of related samples from different classes are as large as possible, while the distances between the sample and the group of related samples in the same class are as small as possible. Therefore the part optimization is:
\begin{align}
\arg\min_{Z_i} \sum_{j=1}^{k_1} \|z_i-z_{i^j}\|_2^2-\kappa\sum_{p=1}^{k_2} \|z_i-z_{i_p}\|_2^2,
\end{align}
where $\kappa$ is a trade-off parameter to control the impacts of the two parts.
Define the coefficient vector:
\begin{align}
\omega_i=\left[\overbrace{1,1,...,1}^{k1},\overbrace{-\kappa,-\kappa,...,-\kappa}^{k2}\right]^T,
\label{equ:part1}
\end{align}
then we can obtain the part optimization matrix,
\begin{align}
L_i=\begin{bmatrix}
      \sum_{j=1}^{k_1+k_2} \left(\omega_i\right)_j & -\omega_i^T \\
      -\omega_i & {\rm diag}\left(\omega_i\right) \\
    \end{bmatrix}.
\label{equ:part2}
\end{align}

\subsection{Whole alignment}

Each patch $X_i$ for $1\leq i\leq n$ has a corresponding low dimensional representation $Z_i$. To unify all low dimensional patches $Z_i=\left[z_i^T,z_{i^1}^T,z_{i^2}^T,\cdots,z_{i^{k_1}}^T,z_{i_1}^T,z_{i_2}^T,\cdots,z_{i_{k_1}}^T\right]^T$ for $1\leq i\leq n$ together into a consistent coordinate system, according to the patch alignment framework, we assume that the coordinate of $Z_i$ is selected from the global coordinate $Z=\left[z_1^T,z_2^T,\cdots,z_n^T\right]^T\in \mathbb R^{n\times d}$ by a using sample selection matrix $S_i\in \mathbb R^{\left(k_1+k_2+1\right)\times n}$:
\begin{align}
Z_i=ZS_i,
\label{equ:Zi}
\end{align}
where the selection matrix $S_i$ is defined by
\begin{align}
\left(S_i\right)_{pq}\left\{
  \begin{array}{ll}
    1, & \hbox{if $q=F_i\left\{p\right\}$;} \\
    0, & \hbox{else.}
  \end{array}.
\right.
\end{align}
According to Eq.\ref{equ:Zi}, the part optimization defined in Eq.\ref{equ:partopt} can be rewritten as:
\begin{align}
\arg\min_{Z} {\rm tr}\left(Z^TS_i^TL_iS_iZ\right).
\end{align}
After summing over all part optimizations together, the whole alignment is given by:
\begin{align}
\notag &\arg\min_{Z} \sum_{i=1}^n{\rm tr}\left(Z^TS_i^TL_iS_iZ\right)\\
\notag =&\arg\min_{Z} {\rm tr}\left(Z^T\sum_{i=1}^n\left(S_i^TL_iS_i\right)Z\right)\\
=&\arg\min_{Z}{\rm tr}\left(Z^TLZ\right),
\end{align}
where $L$ is the alignment matrix. It is obtained by an iterative procedure:
\begin{align}
L\left(F_i,F_i\right)\leftarrow L\left(F_i,F_i\right)+L_i.
\label{equ:whole1}
\end{align}

It is worth emphasizing that the mapping $f:X\mapsto Z$ from the high dimensional space to the low dimensional subspace can be nonlinear and implicit. However, the linear approximation $Z=XW$ is adopted, i.e., we expect the difference between $Z$ and $XW$ is minimized. In particular, $W=\left[w_1,w_2,\cdots,w_d\right]\in \mathbb R^{p\times d}$. Therefore, the objective function is:
\begin{align}
\arg\min_{Z,W} {\rm tr}\left(Z^TLZ\right)+\beta \|Z-XW\|_2^2.
\label{equ:obj1}
\end{align}

\subsection{Classification error minimization}

In MEN, although the discriminative information for classification is considered duly in Eq.\ref{equ:obj1}, the classification error is not directly modeled. To further enhance the performance of MEN for classification problems, it is necessary to provide an explicit way to represent the classification error minimization in the objective function. The least square error minimization is usually adopted in binary classification,
\begin{align}
\arg\min_W \|Y-XW\|_2^2.
\label{equ:errormin}
\end{align}
However, it is very challenging to apply Eq.\ref{equ:errormin} to multi-class classification. This is mainly because the class label vector $C$ cannot be directly utilized as the output (response) $Y$.

Recently, the least squares linear discriminant analysis (Ye \cite{LSLDA}; Sun et al. \cite{LSCCA}) or LS-LDA for short is proposed and presents the equivalence relationship between the least square formulation and the conventional linear discriminant analysis (LDA) for multi-class classification under a mild condition. However, the dimension of the indicator matrix is the number of classes $c$. Therefore, LS-LDA can only reduce the original data to a $c-1$ dimensional subspace. It is pretty fine when samples are drawn from homoscedastic Gaussians because the Bayes optimal is achieved iff the dimension of the subspace is $c-1$. However, for practical applications, samples are usually not sampled from homoscedastic Gaussians and a dozen of experimental evidences show that we usually achieve the best classification performance in a subspace lower than $c-1$ when $c$ is large.

In this paper, we propose a flexible method to design the indicator matrix $Y$ and the dimension of the selected subspace is allowed to be any number between $1$ and $c-1$. In comparing with LS-LDA, the proposed indicator design method is more flexible and powerful to gain a lower dimensional representation and higher recognition rate. Therefore, the new method meets most demands for practical applications, e.g., face recognition.

The nearest-neighbor (NN) rule is commonly applied in classification problmes. In NN, it would be perfect when samples in the same class are projected onto the same point after dimensionality reduction, and this point is the low dimensional representation of the corresponding class center. Meanwhile the variance of these projected class centers is expected to be maximized. As a consequence, the low dimensional projection of class centers can be conveniently obtained by the weighted principal component analysis (PCA).

In detail, suppose the given $n$ samples belong to $c$ classes, and there are $c_i$ samples in the $i^{th}$ class. The $i^{th}$ class center is $o_i=\left(1/c_i\right)\sum\nolimits_{j=1}^{m_i} x_j$, wherein $x_j$ is the $j^{th}$ sample in the $i^{th}$ class and is a row vector in $\mathbb R^p$. The proportion of the $i^{th}$ class is $p_i=c_i/n$. Therefore, the weighted covariance matrix of class centers is given by:
\begin{align}
V=\sum_{i=1}^{m}p_io_i^To_i.
\end{align}
Suppose we expect to find a $d$ dimensional subspace. The $d$ eigenvectors associated with the largest $d$ eigenvalues $\eta=\left[\eta_1,\eta_2,\cdots,\eta_d\right]$ of $V$ are selected to calculate the low dimensional representation of the class center $o_i$ according to
\begin{align}
\hat o_i=o_i\eta.
\label{equ:classerror}
\end{align}
Therefore, the indicator matrix $Y=\left[y_1,y_2,\cdots,y_n\right]^T$ is given by $y_j=\hat o_i$. On combining Eq.\ref{equ:obj1} and Eq.\ref{equ:errormin}, we have
\begin{align}
\arg\min_{Z,W} \|Y-XW\|_2^2+\alpha{\rm tr}\left(Z^TLZ\right)+\beta \|Z-XW\|_2^2,
\label{equ:obj2}
\end{align}
where $\alpha$ and $\beta$ are trade-off parameters to control the impacts of different parts.

\subsection{Elastic net penalty}

In MEN, we expect to obtain a sparse projection matrix for explicit data representation and effective interpretation, i.e., control the number of nonzero elements in each column of the projection matrix. This nonzero number of the entries of the projection matrix can be characterized by the $\ell_0$-norm of the projection matrix. We can impose it over the objective function defined in Eq.\ref{equ:obj2} as a penalty. However, it turns to be an NP-hard problem and thus it is always impossible to be solved in a polynomial time, because the penalty is nonconvex (Lv and Fan \cite{ModelSelectionLv}). Therefore, the $\ell_1$-norm of the projection matrix, i.e., lasso, is usually adopted as a relaxation of the $\ell_0$ penalty. Although lasso is convex, it is difficult to find the solution of the lasso regularized model. This is because the lasso term is not differentiable. Least angle regression or LARS for short has been proposed to greedily search the optimal solution of the lasso penalized linear regression problem. LARS continuously shrinks the particular coefficients (entries of the projection matrix W) towards zeros, while simultaneously preserves high prediction accuracy.

However, the lasso penalty has the following two disadvantages: 1) the number of selected variables is limited by the number of observations and 2) the lasso penalized model can only selects one variable from a group of correlated ones and does not care which one is selected. By imposing an $\ell_2$-norm of the projection matrix on the lasso penalized problem, similar to the elastic net, we can overcome the aforementioned two disadvantages and retain the favorable properties of the lasso penalty. In detail, the $\ell_2$-norm of the projection matrix is helpful to increase the dimension (and the rank) of the combination of the data matrix and the response. In addition, the combination of the $\ell_1$ and $\ell_2$ of the projection matrix is convex with respect to the projection matrix and thus the obtained projection matrix has the grouping effect property.

Therefore, to obtain a sparse projection matrix W with the grouping effect, both $\ell_1$-norm and $\ell_2$-norm of the projection matrix are added as penalties to the objective function defined in Eq.\ref{equ:obj2} and we obtain the full definition of MEN:
\begin{align}
\arg\min_{Z,W} \|Y-XW\|_2^2+\alpha{\rm tr}\left(Z^TLZ\right)+\beta \|Z-XW\|_2^2+\lambda_1\|W\|_1+\lambda_2\|W\|_2^2.
\label{equ:obj3}
\end{align}

\subsection{LARS for MEN}

It has been demonstrated that LARS is effective and efficient to find the optimal solution of the lasso or the elastic net (the combination of $\ell_1$ and $\ell_2$) penalized multiple linear regression. Therefore, it can be directly applied to penalized least squares only. However, the proposed MEN defined in Eq.\ref{equ:obj3}, at the first glance, is not a penalized least square.

In this Section, we detail utilizing LARS to obtain the optimal solution of MEN. Although LARS is designed to solve the penalized multiple linear regression where the coefficients are a vector rather than a matrix, the column vectors of the projection matrix $W$ in MEN are independent bases. Therefore, we can calculate them one by one. In the following analysis, we consider a particular column of $W$, i.e., $w_i$, and the corresponding vector $y_i$ in the indicator matrix $Y$. To simplify the notations below, we keep using $W$ and $Y$ instead of $w_i$ and $y_i$.

Because the low dimensional representation $Z$ and the projection matrix $W$ are independent, we can eliminate $Z$ in the objective function. In detail, $Z$ is obtained by setting the differentiate of the objective function $F$ with respect to $Z$ as $0$, i.e.,
\begin{align}
\frac{\partial F}{\partial Z}=\alpha\left(L+L^T\right)Z+2\beta\left(Z-XW\right)=0.
\end{align}
Therefore, we have
\begin{align}
Z=\beta\left(\alpha L+\beta I\right)^{-1}XW.
\label{equ:Z}
\end{align}
According to Eq.\ref{equ:Z}, we can eliminate $Z$ in the objective function defined in Eq.\ref{equ:obj3}, and thus we have:
\begin{align}
\arg\min_{Z,W} W^TX^TAXW-2W^TX^TY+\lambda_1\|W\|_1+\lambda_2\|W\|_2^2.
\label{equ:obj4}
\end{align}
where this $A$ is an asymmetric matrix computed from $L$:
\begin{align}
\notag A=&\alpha\left(\beta\left(\alpha L+\beta I\right)^{-1}\right)^TL\left(\beta\left(\alpha L+\beta I\right)^{-1}\right)+\\
&\beta\left(\beta\left(\alpha L+\beta I\right)^{-1}-I\right)^T\left(\beta\left(\alpha L+\beta I\right)^{-1}-I\right)+I.
\end{align}
To apply LARS to obtain the optimal solution of Eq.\ref{equ:obj4}, we expect the first item in it to be a quadratic form. Because $2X^TAX=X^T\left(A+A^T\right)X$ and the eigenvalue decomposition of $\left(A+A^T\right)/2$ can be written as $UDU^T$, the objective function defined in Eq.\ref{equ:obj4} without the elastic net penalty can be rewritten as:
\begin{align}
\notag &W^TX^TAXW-2W^TX^TY\\
\notag=&W^TX^T\left(D^{1/2}U^T\right)^T\left(D^{1/2}U^T\right)XW-2W^TX^T\left(D^{1/2}U^T\right)^T\left(\left(D^{1/2}U^T\right)^T\right)^{-1}Y\\
=&\left\|\left(\left(D^{1/2}U^T\right)^T\right)^{-1}Y-\left(D^{1/2}U^T\right)XW\right\|_2^2.
\end{align}
The constant item can be ignored in optimization without loss of generality. We further set
\begin{align}
&X^*=\left(1+\lambda_2\right)^{-1/2}\begin{bmatrix}
                                     \left(D^{1/2}U^T\right)X \\
                                     \sqrt{\lambda_2}I^{p\times p} \\
                                   \end{bmatrix}\in \mathbb R^{\left(n+p\right)\times p}~~{\rm and}\\
\label{equ:newX}
\end{align}
\begin{align}
&Y^*=\begin{bmatrix}
       \left(\left(D^{1/2}U^T\right)^T\right)^{-1}Y \\
       \textbf{0}^{p\times 1} \\
     \end{bmatrix}\in \mathbb R^{\left(n+p\right)\times 1}
\label{equ:newY}
\end{align}
in Eq.\ref{equ:obj4}, and then we get
\begin{align}
\arg\min_{W^*} \|Y^*-X^*W^*\|_2^2+\lambda\|W^*\|_1,
\label{equ:obj5}
\end{align}
where $\lambda=\lambda_1/\left(1+\lambda_2\right)$ and $W^*=\sqrt{1+\lambda_2}W$.

According to Eq.\ref{equ:obj5}, the LARS algorithm can be applied to obtain the optimal solution of the proposed MEN. LARS provides an efficient algorithm to solve the lasso penalized multiple linear regression.

Below we sketch LARS for the transformed MEN defined in Eq.\ref{equ:obj5} and provide novel viewpoints to LARS, which are helpful to better understand the proposed MEN.

We begin with a coefficient vector $W^*$ (a column in the projection matrix with $i^{th}$ entry $\left(W^*\right)_i$ with all zero entries. A variable (a column vector in $X$, i.e., a particular feature) in $\mathbb R^n$ is most correlated with the objective function is added to the active set $A$. Then the corresponding coefficient in $W^*$ increases as large as possible until a second variable (another column vector in $X$, i.e., another feature) in $\mathbb R^n$ has the same correlation as the first variable. Instead of continuously increasing the coefficient vector in the direction of the first variable, LARS proceeds on a direction equiangular over all variables in the active set $A$ until a new variable earns its way into $A$. To make the coefficient vector $W^*$ becomes $K$-sparse (at most $K$ nonzero entries), we conduct the above procedure for $K$ loops. The optimization path direction and the corresponding path length (step size) in LARS are determined by the correlations, which are the negative gradient of the objective function defined in Eq.\ref{equ:obj5} without the lasso penalty, i.e.,
\begin{align}
C=-\frac{\partial F}{\partial W^*}=2\left(X^*\right)^T\left(Y^*-X^*W^*\right)=\left[c_1,c_2,\cdots,c_p\right]^T.
\label{equ:corr}
\end{align}
The constant $2$ can be simply ignored without loss of generality in the following analysis.

The larger the correlation $c_i$ is, the more important the corresponding variable will be, and thus the larger the corresponding coefficient $\left(W^*\right)_i$ in $W^*$ will be. In sparse learning, important variables are added to the active set $A$ sequentially according to their corresponding correlations defined in Eq.\ref{equ:corr}, and then the direction and distance of coefficient vector of all the important variables are determined.

Let $A$ be the active set of "most correlated" variables whose coefficients are nonzero, while the other variables form an inactive set $I$. Thus the sparsity is determined by the cardinality of $A$. The correlations of variables in $A$ are always identical to each other in $A$ and larger than the correlations of variables in $I$. Those correlations of variables in $I$ are usually different to each other. Initially, all the variables are in inactive set $I$ and thus the corresponding coefficients are all zero.

To make $W^*$ $K$-sparse, we need to conduct the following three steps for $K$ loops. In the first step, the variable in the inactive set $I$ with the largest correlation is added to the active set $A$, i.e.,
\begin{align}
\hat C=\max_j \left\{\left|\hat c_j\right|\right\}~~{\rm and}~~A=\left\{j:\left|\hat c_j\right|=\hat C\right\},
\label{equ:addactive}
\end{align}
where $\hat c_j$ is the current correlation of the $j^{th}$ variable.

In the second step, the direction of the coefficient vector $W^*$ is calculated. To make the optimization more global and less greedy, the correlations of the active variables are required to decrease equally in preferred direction. In the $k^{th}$ loop, if the direction vector is $\omega$, then the current correlation is given by
\begin{align}
\notag C_k=&\left(X^*_A\right)^T\left(Y^*-X^*W^*_k\right)\\
\notag =&\left(X^*_A\right)^T\left(Y^*-X^*\left(W^*_{k-1}+\rho\omega\right)\right)\\
=&C_{k-1}+\rho \left(X^*_A\right)^TX^*_A \omega_A,
\label{equ:Ck}
\end{align}
where $X^*_A$ contains all variables in $A$ and each its column is sampled from $X^*$, $C_{k-1}$ is the correlation in the $(k-1)^{th}$ loop, $\rho$ is a constant that is irrelevant to the direction computation, $\omega_A$ stores directions associated with variables in $A$, and the change of the correlation at this step is $\left(X^*_A\right)^TX^*_A \omega_A$. The sign of $\omega_A$, i.e., $s$, is identical to that of $C_{k-1}$, so we can calculate the magnitude of $\omega_A$ directly and then assign its sign as $s$. This $X^*_A\omega_A$ is an extended simplex with vertices defined by active variables. We project the $i^{th}$ column of $X^*$, i.e., $\left(X^*\right)_i$, onto $X^*_A\omega_A$ and thus we get $\left(X^*\right)_i^TX^*_A\omega_A$. Because the correlations of the active variables are required to decrease equally in preferred direction, i.e., $\left(X^*\right)_i^TX^*_A\omega_A$ equals to each other over the index $i$, the only possible solution of $X^*_A\omega_A$ is the normal vector through the origin in the simplex space. Therefore, we have
\begin{align}
\omega_A=s\cdot \left(X^{*T}_AX^*_A\right)^{-1}\textbf{1}_A=s\cdot G_A^{-1}\textbf{1}_A,
\label{equ:rho1}
\end{align}
where $G_A=X^{*T}_AX^*_A$ is the Gram matrix of $X^*_A$. In LARS (Efron et al. 2004), $\omega_A$ is obtained by minimizing the squared distance between the point $X^*_A\omega_A$ on the simplex and the origin, subject to $\|\omega_A\|_1=1$.

To normalize the change of the correlation $X^{*T}_AX^*_A\omega_A$ to a unit vector $u_A$, we need to update $A_A$ and  $\omega_A$, and thus we obtain a normalized $u_A$, i.e.,
\begin{align}
&A_A=\leftarrow \left(\textbf{1}_A^TG_A^{-1}\textbf{1}_A\right)^{-1/2},\\
\label{equ:AA}
&\omega_A\leftarrow s\cdot A_AG_A^{-1}\textbf{1}_A~~{\rm and}\\
\label{equ:omega}
&u_A\leftarrow X^*_A\omega_A.
\end{align}

In the third step, we calculate the distance or magnitude of changes $\rho_1$. To have an efficient optimization procedure, this $\rho_1$ should be as large as possible. At the same time, we have to guarantee that correlations of variables in $A$ are always identical to each other in $A$ and larger than correlations of variables in $I$. Therefore $\rho$ is increased until the correlation of a particular variable in $I$ is equivalent to the correlations of active variables, i.e.,
\begin{align}
\rho_1=\min\nolimits_{j\in A^C}^+\left\{\frac{\hat C-\hat c_j}{A_A-a_j},\frac{\hat C+\hat c_j}{A_A+a_j}\right\},
\label{equ:rho1}
\end{align}
where $A_C$ is the complement of $A$, $a=X^{*T}_Au_A$, $a_j$ is the $j^{th}$ entry of $a$, $\hat C$ is the largest correlation defined in Eq.\ref{equ:addactive} and obtained in the first step, and $\rho_1$ is a possible candidate of $\rho$ mentioned in Eq.\ref{equ:Ck}.

According to LARS, to obtain an identical solution to MEN defined in Eq.\ref{equ:obj5}, the lasso modification is considered, i.e., the argument of the distance $\rho$ stops increasing when a coefficient of variables in $A$ is zero, or mathematically,
\begin{align}
W^*_{Ak}=W^*_{Ak-1}+\rho_2s_A\omega_A=0,
\label{equ:calrho2}
\end{align}
where $\rho_2$ is another possible candidate of $\rho$ defined in Eq.\ref{equ:Ck}. According to Eq.\ref{equ:calrho2}, we can obtain
\begin{align}
\rho_2=\min\nolimits^+\left\{-W^*_{Ak-1}/s_A\omega_A\right\}.
\label{equ:rho2}
\end{align}
Therefore, the distance of $W^*$, i.e., $\rho$, is the minimum of $\rho_1$ and $\rho_2$, i.e.,
\begin{align}
\rho=\min\nolimits^+\left\{\rho_1,\rho_2\right\}.
\label{equ:rho}
\end{align}

In each loop, one new variable is added to the active set $A$ according to Eq.\ref{equ:addactive}, the direction and distance of the coefficient vector $W^*$ are calculated according to Eq.\ref{equ:omega} and Eq.\ref{equ:rho}. After $K$ loops, $W^*$ is $K$-sparse. According to the elastic net, to eliminate the double shrinkage, the optimal $W$ should be corrected:
\begin{align}
W=\sqrt{1+\lambda_2}W^*.
\label{equ:updateW}
\end{align}

\subsection{Fast LARS}

LARS is inefficient when the size of the training set is large, because the time cost for calculating the inverse of the Gram matrix $G_A$ defined in Eq.\ref{equ:rho1} is huge. Because the dimension of this $G_A$ is increasing at each of the $K$ loops, according to (Golub and Van Loan \cite{MatrixComputation}), the inverse of $G_A$ can be obtained incrementally, i.e., the inverse of the Gram matrix $\left(G_{A_k}\right)^{-1}$ in the $k^{th}$ loop can be updated from $\left(G_{A_{k-1}}\right)^{-1}$ in the previous loop. Particularly, in the $k^{th}$ loop, a new variable $\left(X\right)_i\in \mathbb R^n$ is added to the active set $A$, and thus we have
\begin{align}
\notag G_{A_k}=&X^{*T}_{A_k}X^{*}_{A_k}=X^{T}_{A_k}X_{A_k}+2\lambda_2I\\
\notag =&\begin{bmatrix}
    X^{T}_{A_{k-1}} \\
    \left(X\right)_i^T \\
  \end{bmatrix}\begin{bmatrix}
                 X_{A_{k-1}} & \left(X\right)_i \\
               \end{bmatrix}+2\lambda_2I\\
\notag =&\begin{bmatrix}
           X^{T}_{A_{k-1}}X_{A_{k-1}} & X^{T}_{A_{k-1}}\left(X\right)_i \\
           \left(X\right)_i^TX_{A_{k-1}} & \left(X\right)_i^T\left(X\right)_i \\
         \end{bmatrix}+2\lambda_2I\\
=&\begin{bmatrix}
    X^{T}_{A_{k-1}}X_{A_{k-1}}+2\lambda_2I & X^{T}_{A_{k-1}}\left(X\right)_i \\
    \left(X\right)_i^TX_{A_{k-1}} & \left(X\right)_i^T\left(X\right)_i+2\lambda_2 \\
  \end{bmatrix}.
\end{align}

Let $A$, $B$, $C$ and $D$ be the blocks of $G_A$, i.e., $A=X^{T}_{A_{k-1}}X_{A_{k-1}}+2\lambda_2I$, $B=X^{T}_{A_{k-1}}\left(X\right)_i$, $C=\left(X\right)_i^TX_{A_{k-1}}$, and $D=\left(X\right)_i^T\left(X\right)_i+2\lambda_2$. Let $S_A$ to be the Schur complement of A, i.e., $S_A=D-CA^{-1}B$. According to rules of the block matrix calculation, $\left(G_{A_k}\right)^{-1}$ is given by:
\begin{align}
\left(G_{A_k}\right)^{-1}=\begin{bmatrix}
                            A^{-1}+A^{-1}BS_A^{-1}CA^{-1} & -A^{-1}BS_A^{-1} \\
                            -S_A^{-1}CA^{-1} & S_A^{-1} \\
                          \end{bmatrix},
\label{equ:FastLARS}
\end{align}
where $A^{-1}=\left(G_{A_{k-1}}\right)^{-1}$ is the inverse of the Gram matrix obtained in the previous loop. The time cost for calculating the inverse of the Gram matrix in the $k^{th}$ loop can be reduced from $\mathcal O(p^3)$ to $\mathcal O(p^2+5p)$ ($p$ is the size of active set in the $k^{th}$ loop) when the inverse of the Gram matrix in the previous loop is available.

We can further accelerate the computation of LARS for MEN by taking the advantage of the sparse structure of $X^*$. For example, when calculating the equiangular vector $a$ and the inner product $G_A$, the block matrix calculation can reduce the time cost as well.

\subsection{Algorithm}

In this paper, we propose an efficient framework MEN for discriminative dimensionality reduction with sparse projection. Based on the discussion in the above six subsections, MEN is shown in Algorithm 1.

\begin{algorithm}[h]
\begin{algorithmic}
\STATE \textbf{Input:} ~~~~~~Training data matrix $X=\left[x_1,x_2,\cdots,c_n\right]\in \mathbb R^{n\times p}$;\\
~~~~~~~~~~~~~~~~~~Class label vector $C=\left[c_1,c_2,\cdots,c_n\right]^T$;\\
~~~~~~~~~~~~~~~~~~$W=\left[w_1,w_2,\cdots,w_d\right]\in \textbf{0}^{p\times d}$, where $d$ is the dimensions of subspace;\\
~~~~~~~~~~~~~~~~~~The number of loops $K$, small $K$ induces sparser $W$.
\STATE \textbf{Output:} ~~~Sparse projection matrix $W=\left[w_1,w_2,\cdots,w_d\right]\in \mathbb R^{p\times d}$.
\STATE \textbf{Initialize:} $k:=0$.
\REPEAT
\STATE Step 1: Optional PCA reconstruction of original data $X$.
\STATE Step 2: Part optimization: build $n$ patches for the $n$ given samples according to definition of\\
~~~~~~~~~~~~manifold, calculate matrix $L_i$ for each patch using Eq.\ref{equ:part1} and Eq.\ref{equ:part2}.
\STATE Step 3:Whole alignment: unify the patches in a global coordinate, compute big matrix $L$\\
~~~~~~~~~~~~using Eq.\ref{equ:whole1}.
\STATE Step 4: Classification error minimization: Calculate the indicator matrix $Y$ using scaled\\
~~~~~~~~~~~~PCA for class centers using Eq.\ref{equ:classerror}.
\STATE Step 5: New data matrix and indicator matrix: Calculate $X^*$ and $Y^*$ from $X$ and $Y$ using\\
~~~~~~~~~~~~Eq.\ref{equ:newX} and Eq.\ref{equ:newY}.
\STATE Step 6: Column by column loops for $W$,$k:=k+1$.
\STATE \textbf{Initialize:} $m:=0$.
\REPEAT
\STATE $m:=m+1$.
\STATE Update active set: add the variable with largest correlation to $A$ using Eq.\ref{equ:corr} and Eq.\ref{equ:addactive}.
\STATE Direction calculation using Eq.\ref{equ:AA}, Eq.\ref{equ:omega} and fast LARS Eq.\ref{equ:FastLARS}.
\STATE Distance calculation using Eq.\ref{equ:rho1}, Eq.\ref{equ:rho2} and Eq.\ref{equ:rho}.
\STATE Update $w_k$ using Eq.\ref{equ:updateW}.
\UNTIL m=K.
\STATE Step 7: Update projection matrix $W$ by adding $w_k$ into $W$.
\UNTIL k=d.
\STATE \textbf{return} $W$.
\end{algorithmic}
\caption{Manifold Elastic Net (MEN)}
\end{algorithm}

In MEN, after necessary initializations, we first build patches for all training samples by calculating $L_i$ of each patch in the part optimization according to Eq.\ref{equ:part2} in subsection 1. Then these $L_i$ matrixes are unified in a global coordinate system into one matrix $L$ according to Eq.\ref{equ:whole1} in whole alignment step explained in subsection 2. Afterwards, the indicator matrix $Y$ is computed according to the weighted PCA over class centers defined in Eq.\ref{equ:classerror} in subsection 3. A matrix $A$ defined in Eq.\ref{equ:obj3} in the objective function can be obtained from $L$ and other parameters. The eigenvalue decomposition is conducted over $\left(A+A^T\right)/2$ to construct the new data matrix $X^*$ and the new indicator matrix $Y^*$ according to Eq.\ref{equ:newX} and Eq.\ref{equ:newY}, respectively.

Then the LARS algorithm is applied to calculate a sparse projection matrix. The direction and distance of each loop are computed according to Eq.\ref{equ:omega} and Eq.\ref{equ:rho}. The incremental method to obtain the inverse of the Gram matrix defined in Eq.\ref{equ:FastLARS} is considered speeding up LARS. This process is conducted several times and the projection matrix is computed column by column. Finally a sparse projection matrix is obtained as the output of MEN. This matrix is ready to project a given sample in $\mathbb R^p$ to a low dimensional subspace $\mathbb R^d$ with $K$-sparse.

MEN is an efficient algorithm with high convergence velocity, because the computation in LARS explained in subsections 5 and 6 is equivalent to the cost of a least square fit. Given a training set $X\in \mathbb R^{n\times p}$, to obtain a sparse matrix $W\in \mathbb R^{p\times d}$ each column of which contains K nonzero elements, d times of LARS are required in MEN. Most steps in LARS are simple matrix computations. For $p\gg n$, MEN requires $\mathcal O(dK^3+dpK^2)$ operations.

\subsection{Discussions}

MEN integrates the merits of both manifold learning and sparse learning via a unified framework. It is not a direct combination of these two popular learning schemes but a complementary embedding of both. Through the patch alignment framework, the local geometry of a given dataset is retained in MEN. The weighted lasso and $\ell_2$ penalties are added to produce a sparse projection matrix with the grouping effect. The combined lasso and $\ell_2$ is also termed as the elastic net. Therefore, we term the proposed framework as the manifold elastic net. As a consequence, MEN is superior to existing dimensionality reduction algorithms, because of its powerful variable selection function and consideration of the intrinsic structure of the dataset.

It has been well demonstrated that LARS is effective and efficient to solve a lasso regularized least square problem. Therefore, to apply LARS to find the optimal solution of MEN, it is essential to prove that MEN is equivalent to a lasso regularized least square problem and LARS converges for optimization. In particular, we prove that LARS can optimize a general form of the lasso regularized problem, which contains both MEN and the lasso regularized least square problem as special cases.

\begin{theorem}
LARS can solve a general form of the lasso regularized problem defined below:
\begin{align}
\arg\min_\beta \beta^TA\beta+\beta^TB+C+t\|\beta\|_1,
\label{equ:theo1}
\end{align}
where $\beta\in \mathbb R^{p\times 1}$ and $A\in \mathbb R^{p\times p}$ (could be an asymmetric square matrix), $B\in \mathbb R^{p\times 1}$, and $C$ and $t$ are constants.
\end{theorem}
\begin{proof}
It is equivalent to prove that the problem defined in Eq.\ref{equ:theo1} is equivalent to a lasso regularized least square problem.

The objective function defined in Eq.\ref{equ:theo1} without the lass penalty can be written as:
\begin{align}
\beta^TA\beta+\beta^TB+C=\beta^T\left(\frac{A+A^T}{2}\right)\beta+\beta^TB+C,
\end{align}
where $\left(A+A^T\right)/2\in \mathbb R^{p\times p}$ is a symmetric matrix and its eigenvalue decomposition is $\left(A+A^T\right)/2=UDU^T$.

Therefore, we have:
\begin{align}
\notag &\beta^T\left(\frac{A+A^T}{2}\right)\beta+\beta^TB+C\\
\notag =&\beta^T\left(D^{1/2}U^T\right)^T\left(D^{1/2}U^T\right)\beta-\\
\notag &2\beta^T\left(D^{1/2}U^T\right)^T\left(-\frac{1}{2}\left(\left(D^{1/2}U^T\right)^T\right)^{-1}B\right)+C\\
=&\left\|\left(-\frac{1}{2}\left(\left(D^{1/2}U^T\right)^T\right)^{-1}B\right)-\left(D^{1/2}U^T\right)\beta\right\|_2^2+{\rm const}.
\end{align}
To simply represent the above objective function, without loss of generality, let
\begin{align}
Y=-\frac{1}{2}\left(\left(D^{1/2}U^T\right)^T\right)^{-1}B,X=\left(D^{1/2}U^T\right),
\end{align}
and ignore the constant. Therefore, we can transform the problem defined in Eq.\ref{equ:theo1} to
\begin{align}
\arg\min_\beta \left\|Y-X\beta\right\|_2^2+t\|\beta\|_1,
\end{align}
which is a lasso regularized least square problem. It is not difficult to prove that MEN is a special case of the problem defined in Eq.\ref{equ:theo1}. Therefore, LARS can be applied to solve MEN and the problem defined in Eq.\ref{equ:theo1}.
\end{proof}

\begin{theorem}
LARS converges in optimizing the problem defined in Eq.\ref{equ:theo1} in Theorem 1.
\end{theorem}
\begin{proof}
Let the objective function defined in Eq.\ref{equ:theo1} without the lasso penalty be $F$. After the $k^{th}$ loop, assume the estimate of the objective function becomes $F_k$. If $F$ is smooth in each loop, we have:
\begin{align}
\notag \frac{F_k-F_{k-1}}{\omega_i}\in &\left[\min\left\{\frac{\partial F_k}{\partial \beta_i}\left|_{\beta_i=\beta_i^k}\right.,\frac{\partial F_k}{\partial \beta_i}\left|_{\beta_i=\beta_i^{k-1}}\right.\right\},\right.\\
&\left.\max\left\{\frac{\partial F_k}{\partial \beta_i}\left|_{\beta_i=\beta_i^k}\right.,\frac{\partial F_k}{\partial \beta_i}\left|_{\beta_i=\beta_i^{k-1}}\right.\right\}\right],
\end{align}
where $\beta_i$ is the $i^{th}$ element in coefficient vector $\beta$, and $\omega$ is the change of $\beta$ between two consecutive loops, i.e., $\omega=\beta^k-\beta^{k-1}=\left[\omega_1,\omega_2,\cdots,\omega_p\right]^T$.

In LARS for the problem defined in Eq.\ref{equ:theo1}, the sign of $\omega$ is the negative gradient of objective function $F$ on $\beta^{k-1}$, i.e.,
\begin{align}
{\rm sign}\left(\omega_i\right)={\rm sign}\left(-\frac{\partial F_k}{\partial \beta_i}\left|_{\beta_i=\beta_i^{k-1}}\right.\right).
\end{align}

In each loop of LARS, when correlation of one active variable becomes zeros, the length of the coefficient path will stop increasing. Therefore, the sign vector of correlations will not change in one loop, i.e.,
\begin{align}
\notag {\rm sign}\left(-\frac{\partial F_k}{\partial \beta_i}\left|_{\beta_i=\beta_i^{k}}\right.\right)={\rm sign}\left(-\frac{\partial F_k}{\partial \beta_i}\left|_{\beta_i=\beta_i^{k-1}}\right.\right)={\rm sign}\left(\frac{F_k-F_{k-1}}{\omega_i}\right)=-{\rm sign}\left(\omega_i\right)
\end{align}
According to the analyses, we can obtain the sign of $\left(F_k-F_{k-1}\right)$:
\begin{align}
{\rm sign}\left(F_k-F_{k-1}\right)=-{\rm sign}\left(\omega\right)\cdot{\rm sign}\left(\omega\right)=-1.
\end{align}

According to the above equation, the objective function $F$ is monotonic. In addition, $F$ is bounded. Therefore, we can safely draw the conclusion that LARS converges in optimizing the problem defined in Eq.\ref{equ:theo1}.
\end{proof}

\section{Experiments}

In this section, we evaluate the performance of MEN by comparing against six representative dimensionality reduction algorithms, i.e., principal component analysis (PCA), Fisher's linear discriminant analysis (FLDA), discriminative locality alignment (DLA) (Zhang et al. \cite{DLAeccv}; Zhang et al. \cite{DLAtkde}), supervised locality preserving projection (SLPP), neighborhood preserving embedding (NPE), and sparse principal somponent analysis (SPCA), on three standard face image databases, i.e., UMIST (Graham and Allinson \cite{UMIST}), FERET (Phillips et al. \cite{FERET}) and YALE (Belhumeur et al. \cite{fisherface}).

PCA is an unsupervised linear dimensionality reduction algorithm which projects the data along the direction of maximal variance. FLDA is a supervised linear dimensionality reduction method. SLPP is a supervised modification of the locality preserving projections, which is a linearization of the Laplacian Eigenmap. NPE is a linear approximation to the locally linear embedding (LLE). SPCA is a sparse dimensionality reduction algorithm which combines the lasso penalty with PCA to produce sparse loadings.

\begin{figure}[t]
\begin{center}
 \includegraphics[width=1\linewidth]{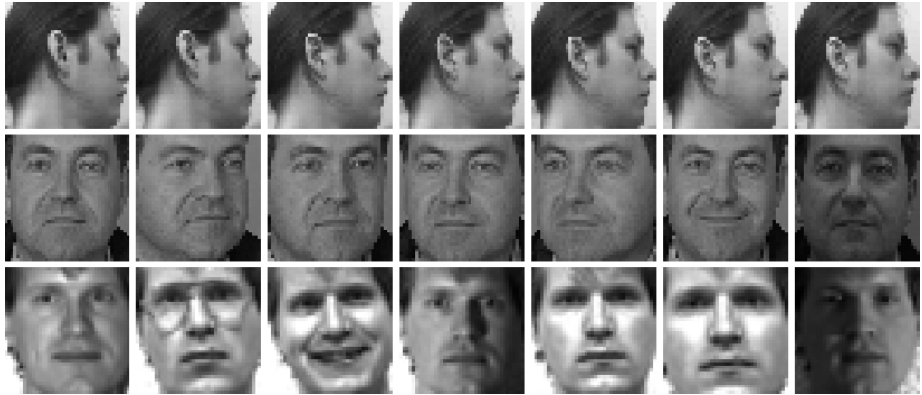}
\end{center}
   \caption{Sample face images from the three databases. The first row comes from UMIST; the second row comes from FERET; and the third row comes from YALE.}
\label{fig:sampleface}
\end{figure}

Three standard face image datasets, e.g., UMIST, FERET and YALE, are utilized in this paper to evaluate the proposed MEN for discriminative dimensionality reduction. There are $565$ face images from $20$ individuals in the UMIST dataset. The samples demonstrate variations in race, gender, pose and appearance. The FERET dataset consists of $13,539$ face images from $1,565$ individuals. The images vary in size, gender, pose, illumination, facial expression and age. We randomly select $100$ individuals, each of which has $7$ images from FERET for performance evaluation. The YALE dataset contains $165$ face images of $15$ individuals. Lighting conditions, gender, facial expressions and configurations are different among these images. All images from these three databases are normalized to $40\times 40$ pixel arrays with $256$ gray levels per pixel. Fig.\ref{fig:sampleface} shows sample images from these three datasets. Each image is reshaped to a long vector by concatenating its pixel values in a particular order.

Different algorithms follow an equivalent procedure for all face recognition experiments on various datasets. Firstly, the database is randomly divided into two separate sets: training set and testing set. Then the training set is used to learn the low dimensional subspace and corresponding projection matrix through given algorithm. After this, samples in the testing set are projected to a low dimensional subspace via the projection matrix. Finally, the nearest neighbor classifier is used to recognize testing samples in the subspace.

We apply PCA to reduce dimensions of original high dimensional face images before FLDA, DLA, LPP (with supervised setting) and NPE (with supervised setting). For FLDA, we retain $n-c$ dimensions in the PCA projection, where $n$ is the number of samples and $c$ is the number of classes. We project samples to the PCA subspace with $n-1$ dimensions for DLA, SLPP and NPE.

For UMIST and YALE, we randomly select $p=\left(5,7\right)$ images per individual for training, while the remaining images are used as testing samples. For FERET, $p=\left(4,5\right)$ images per individual are selected as training set, and the remaining for testing. All experiments are repeated five times, and the average recognition rates are calculated.

\begin{figure}[t]
\begin{center}
 \includegraphics[width=1\linewidth]{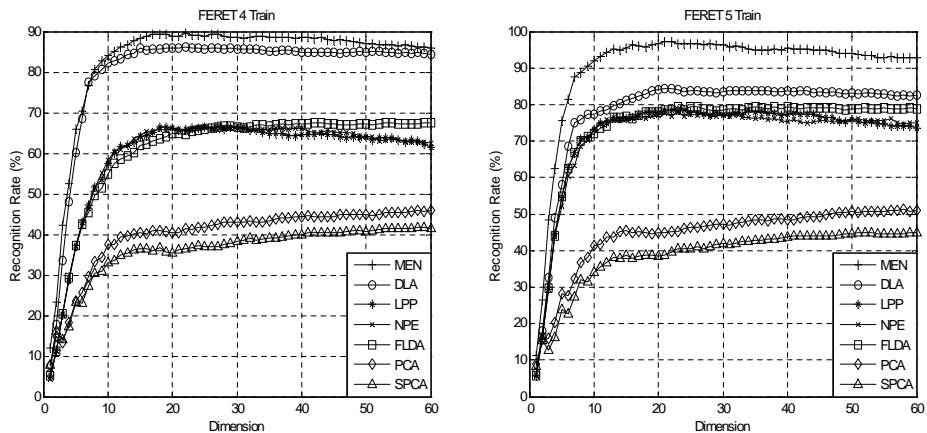}
\end{center}
   \caption{Recognition Rate vs. Dimension on FERET}
\label{fig:rateFERET}
\end{figure}

The results of these dimensionality reduction algorithms on two settings of FERET are shown in Fig.\ref{fig:rateFERET}. These seven algorithms can be divided into 3 groups according to their performance: PCA and SPCA are at the bottom level, because they are unsupervised and the label information is not considered. PCA is slightly better than SPCA, because SPCA is designed to approximate PCA but with less information retained to hold the sparse property. LPP, NPE and LDA are at the middle level. They are much better than PCA and SPCA because they consider the class label information. LPP and NPE preserve the local geometry based on the neighborhood information of samples, while LDA ignores the local geometry. LPP and NPE cannot perform as well as DLA and MEN because both of them ignore the margin maximization or the inter-class information. MEN and DLA are at the top level. MEN outperforms DLA because it reduces the noises by using the elastic net penalty.

\begin{figure}[t]
\begin{center}
 \includegraphics[width=1\linewidth]{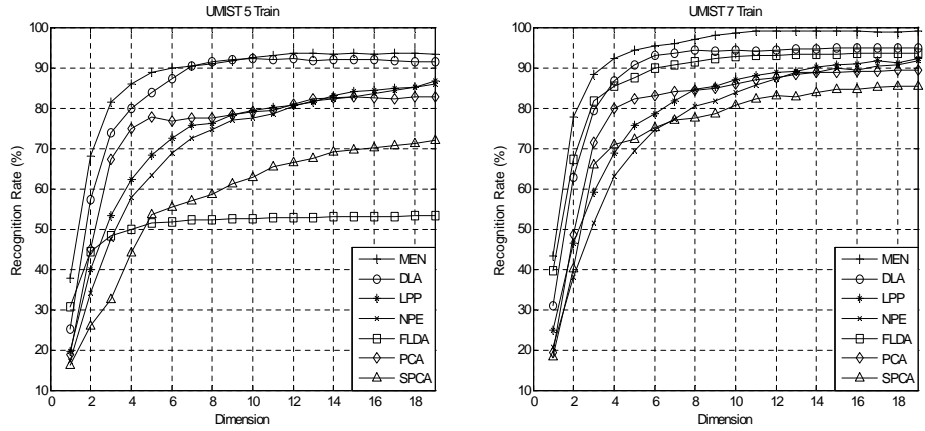}
\end{center}
   \caption{Recognition Rate vs. Dimension on UMIST}
\label{fig:rateUMIST}
\end{figure}

Experimental results on UMIST are shown in Fig.\ref{fig:rateUMIST}. MEN outperforms the other six algorithms consistently. Note the fact that MEN keeps having the highest recognition rate when the dimension of the selected subspace is low. This verifies the robustness of MEN in low dimension situation. In addition, the computational cost is proportional to the dimension of the selected subspace. Therefore MEN produces better results with less computational cost than other dimensionality reduction methods.

\begin{figure}[t]
\begin{center}
 \includegraphics[width=1\linewidth]{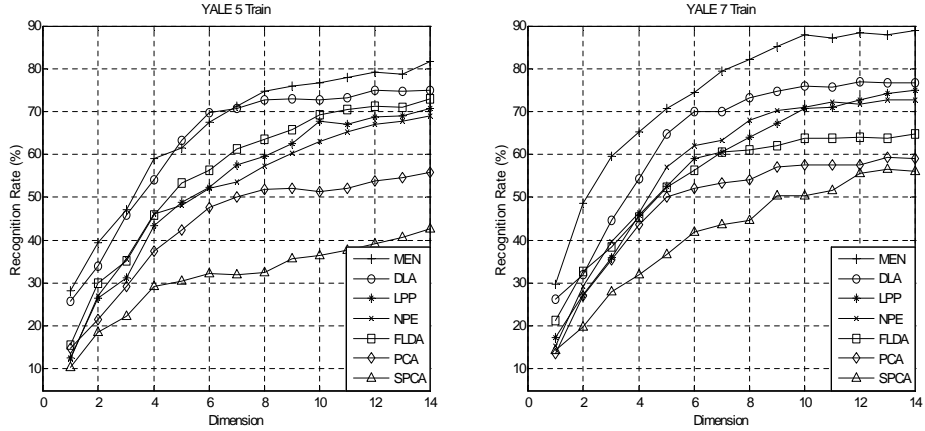}
\end{center}
   \caption{Recognition Rate vs. Dimension on YALE}
\label{fig:rateYALE}
\end{figure}

Fig.\ref{fig:rateYALE} shows MEN outperforms the other six algorithms on the YALE dataset. The curves of MEN are smoother than those of the other algorithms. This implicates that MEN is more stable than the other algorithms. MEN has high recognition rate even when the training set is small and the dimensions of the selected subspace is low. The priority of MEN can be attributes to its supervised learning property, consideration of data manifold structure, feature selection ability brought by sparsity and the grouping effect. The sparsity of MEN filters out classification irrelevant features, which bring unnecessary noises for classification. This is effective especially when the number of classes is much smaller than the number of the original features. Furthermore, the sparse projection matrix brings better interpretation and lower computational cost for subsequent calculation than dense projection matrices.

Table \ref{table:recognitionrate} lists the best recognition rate and the corresponding subspace dimension for each algorithm in the experiments on the three face image datasets. Sparse dimensionality reduction algorithm including MEN and SPCA always arrive their best recognition rate in lower dimensional subspace than other five algorithms. This is because the sparsity brought by the lasso penalty is able to select the most significant features. However, because SPCA does not consider the class label information, it always performs more poorly than other supervised algorithms. For each algorithm, the dimension of the best recognition rate is decreasing with the increasing of training samples. This is because more training samples make the low dimensional representation more stable and reliable.

\begin{table}
\begin{center}
\begin{tabular}{|c|c|c|c|c|c|c|c|c|}
\hline
\multicolumn{2}{|c|}{} & MEN & DLA & LPP & NPE & LDA & PCA & SPCA \\
\hline
\multirow{2}{*}{FERET}  & 4 & 90.67(17) & 88.67(19) & 74.00(17) & 74.33(21) & 76.33(25) & 48.00(54) & 45.67(41)\\ \cline{2-9}
                        & 5 & 96.50(30) & 88.50(35) & 83.50(36) & 82.00(19) & 84.00(49) & 54.00(51) & 48.50(58)\\
\hline \hline
\multicolumn{2}{|c|}{} & MEN & DLA & LPP & NPE & LDA & PCA & SPCA \\
\hline
\multirow{2}{*}{UMIST}  & 5 & 95.89(17) & 94.57(18) & 90.11(19) & 89.68(19) & 88.21(18) & 88.63(13) & 80.63(19)\\
\cline{2-9}
                        & 7 & 99.21(16) & 97.62(19) & 95.40(19) & 95.17(18) & 97.24(14) & 93.79(19) & 90.57(18)\\
\hline \hline
\multicolumn{2}{|c|}{} & MEN & DLA & LPP & NPE & LDA & PCA & SPCA \\
\hline
\multirow{2}{*}{YALE}   & 5 & 82.78(13) & 79.11(12) & 79.33(13) & 77.11(14) & 82.22(12) & 61.11(12) & 63.33(13)\\
\cline{2-9}
                        & 7 & 90.33(12) & 87.00(12) & 85.00(13) & 84.33(11) & 81.67(11) & 66.67(13) & 63.33(12)\\
\hline
\end{tabular}
\end{center}
\caption{Best recognition rate (\%) on three databases. For MEN, DLA, LPP (SLPP), NPE, LDA (FLDA), PCA, SPCA (Sparse PCA), the numbers in the parentheses behind the recognition rates are the subspace dimensions. Numbers in the second column denote the number of training samples per individual.}
\label{table:recognitionrate}
\end{table}

Boxplots of the experimental results of these seven dimensionality reduction algorithms on the three face image datasets are shown in Fig.\ref{fig:boxFERET}, Fig.\ref{fig:boxUMIST} and Fig.\ref{fig:boxYALE}, respectively. Each boxplot produces a box and whisker plot for each method. The box has lines at the lower quartile, median, and upper quartile values. Whiskers extend from each end of the box to the adjacent values in the data-by default and the most extreme values within $1.5$ times the interquartile range from the ends of the box.

MEN achieves the most robust recognition rate, because it considers the sparse property, the local geometry of intra-class samples, and the margin maximization and classification error minimization of inter-class samples. MEN selects features with the largest correlation and eliminates the most unstable ones. Manifold learning methods, such as LPP, DLA and NPE, as well as LDA are more stable than PCA and SPCA according to these boxplots because they consider the class label information.

\begin{figure}[htdp]
\begin{center}
 \includegraphics[width=1\linewidth]{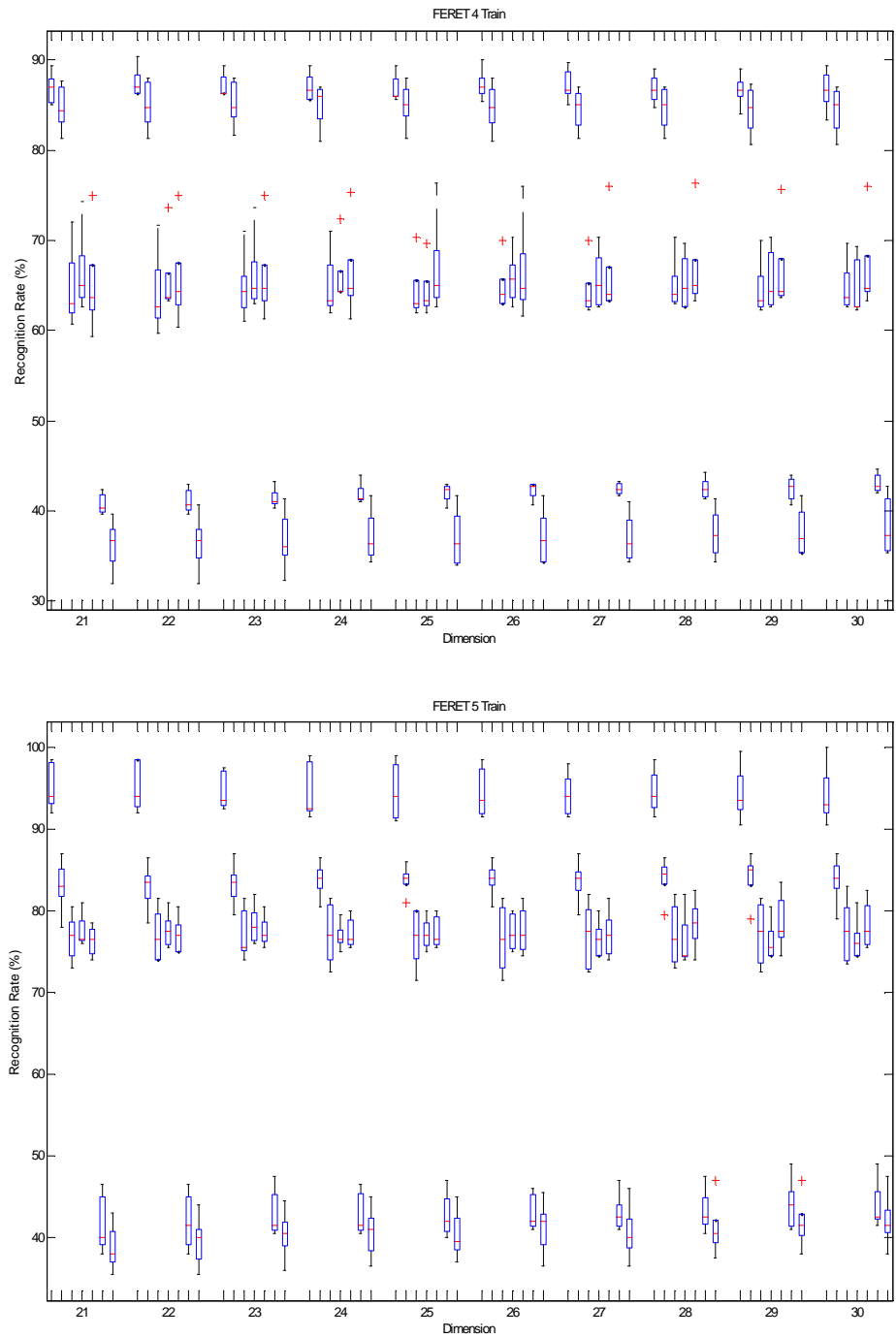}
\end{center}
   \caption{Boxplot of Recognition Rate vs. Dimension (from $21$ to $30$) on FERET with $4$ ($5$) training samples per person. For every dimension, from left to right, the seven boxes refer to MEN, DLA, LPP, NPE, FLDA, PCA, and SPCA.}
\label{fig:boxFERET}
\end{figure}

\begin{figure}[htdp]
\begin{center}
 \includegraphics[width=1\linewidth]{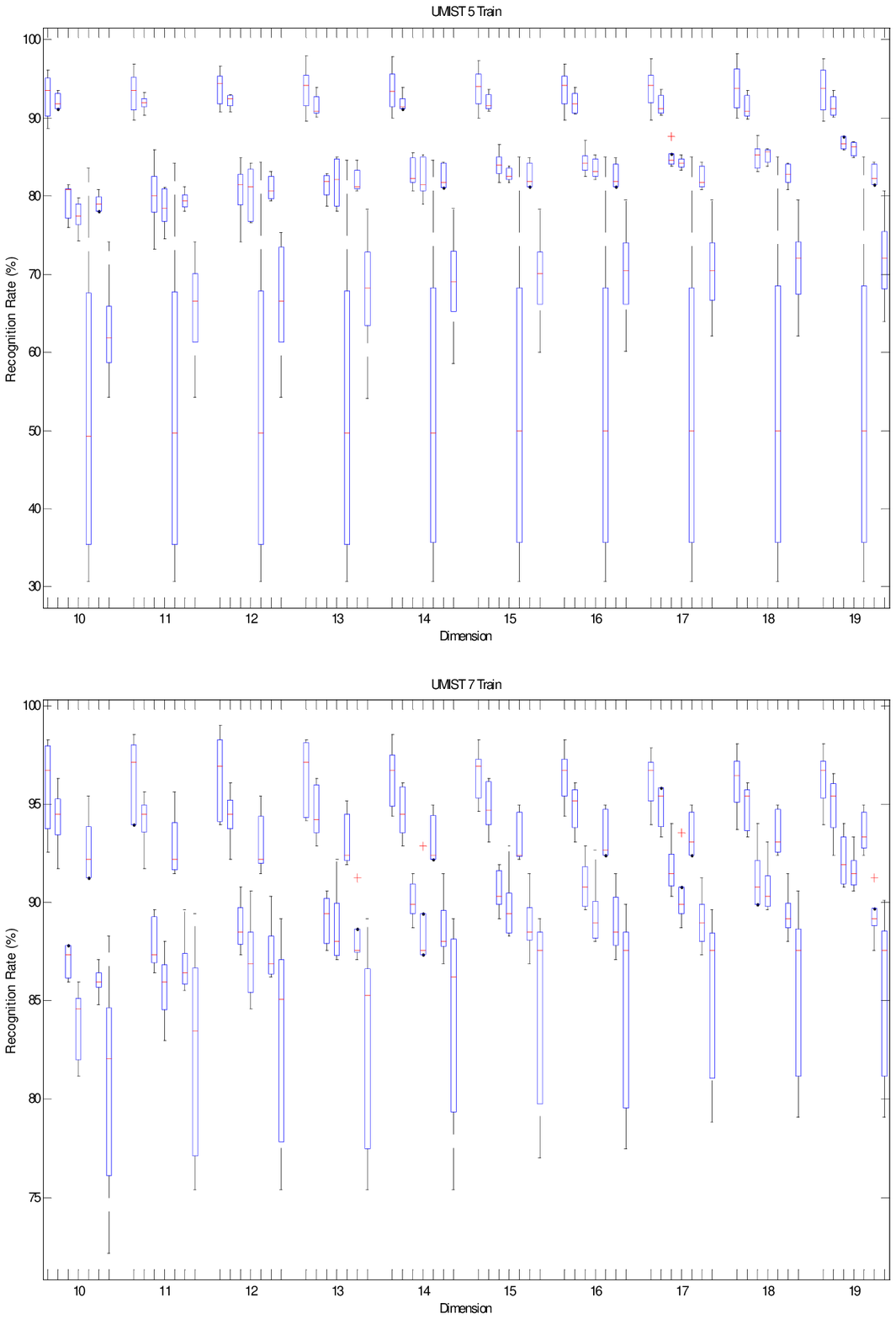}
\end{center}
   \caption{. Boxplot of Recognition Rate vs. Dimension (from $10$ to $19$) on UMIST with $5$ ($7$) training samples per person. For every dimension, from left to right, the seven boxes refer to MEN, DLA, LPP, NPE, FLDA, PCA, and SPCA.}
\label{fig:boxUMIST}
\end{figure}

\begin{figure}[htdp]
\begin{center}
 \includegraphics[width=1\linewidth]{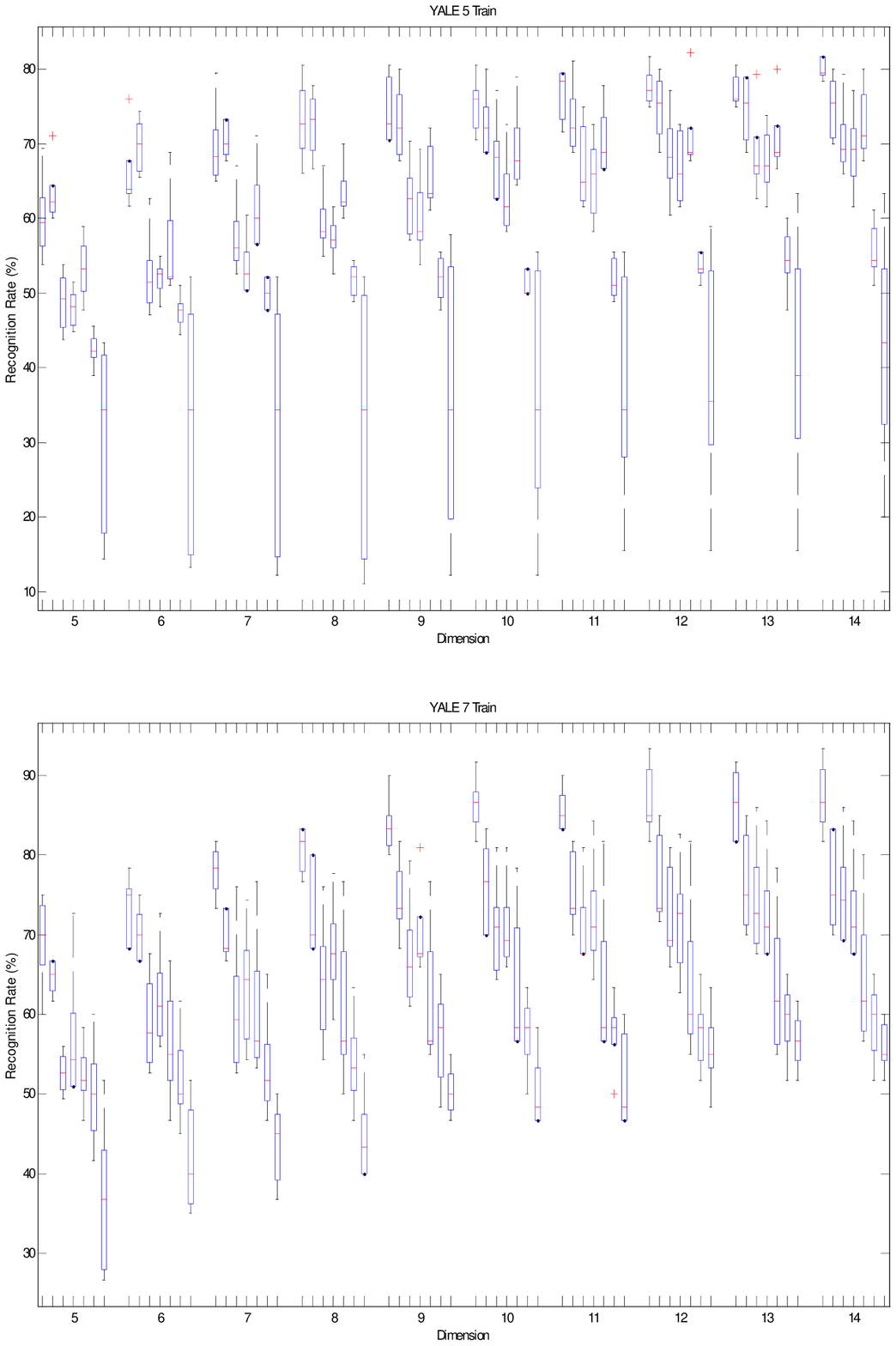}
\end{center}
   \caption{Boxplot of Recognition Rate vs. Dimension (from $5$ to $14$) on YALE with $5$ ($7$) training samples per person. For every dimension, from left to right, the seven boxes refer to MEN, DLA, LPP, NPE, FLDA, PCA, and SPCA.}
\label{fig:boxYALE}
\end{figure}

Fig.\ref{fig:faceFERET}, Fig.\ref{fig:faceUMIST} and Fig.\ref{fig:faceYALE} show the columns of the projection matrix   of the seven algorithms on the three face image datasets. The low dimensional subspace is spanned by the column vectors, which is called bases. The bases of PCA are called Eigenfaces (Turk and Pentland \cite{Eigenface}), while the bases of LDA are called Fisherfaces (He et al. \cite{LaplacianFace}) in previous literatures. Similar methods can be applied to DLA, SLPP, NPE, SPCA and MEN. The bases of MEN are sparser and have less noise than PCA and DLA because of its sparsity, and more grouping than SPCA because of its grouping effect adopted from the $\ell_2$ penalty. Sparse bases lead to computational efficiency and good interpretation. According to Fig.\ref{fig:faceFERET}, Fig.\ref{fig:faceUMIST} and Fig.\ref{fig:faceYALE}, ``MEN faces'' retain the most discriminative facial features, e.g., eyebrows, eyes, nose, mouth, ears and facial contours, while leave the other parts blank. ``SPCA faces'' are sparse but without the grouping effect, their facial contours and organs are represented by some isolate points. ``LPP faces'' and ``NPE faces'' are very similar in appearances and this fact well explains that they perform comparably in these datasets. ``DLA faces'' have better description of features and less noises than those obtained by LPP, NPE and FLDA.

\begin{figure}[htdp]
\begin{center}
 \includegraphics[width=1\linewidth]{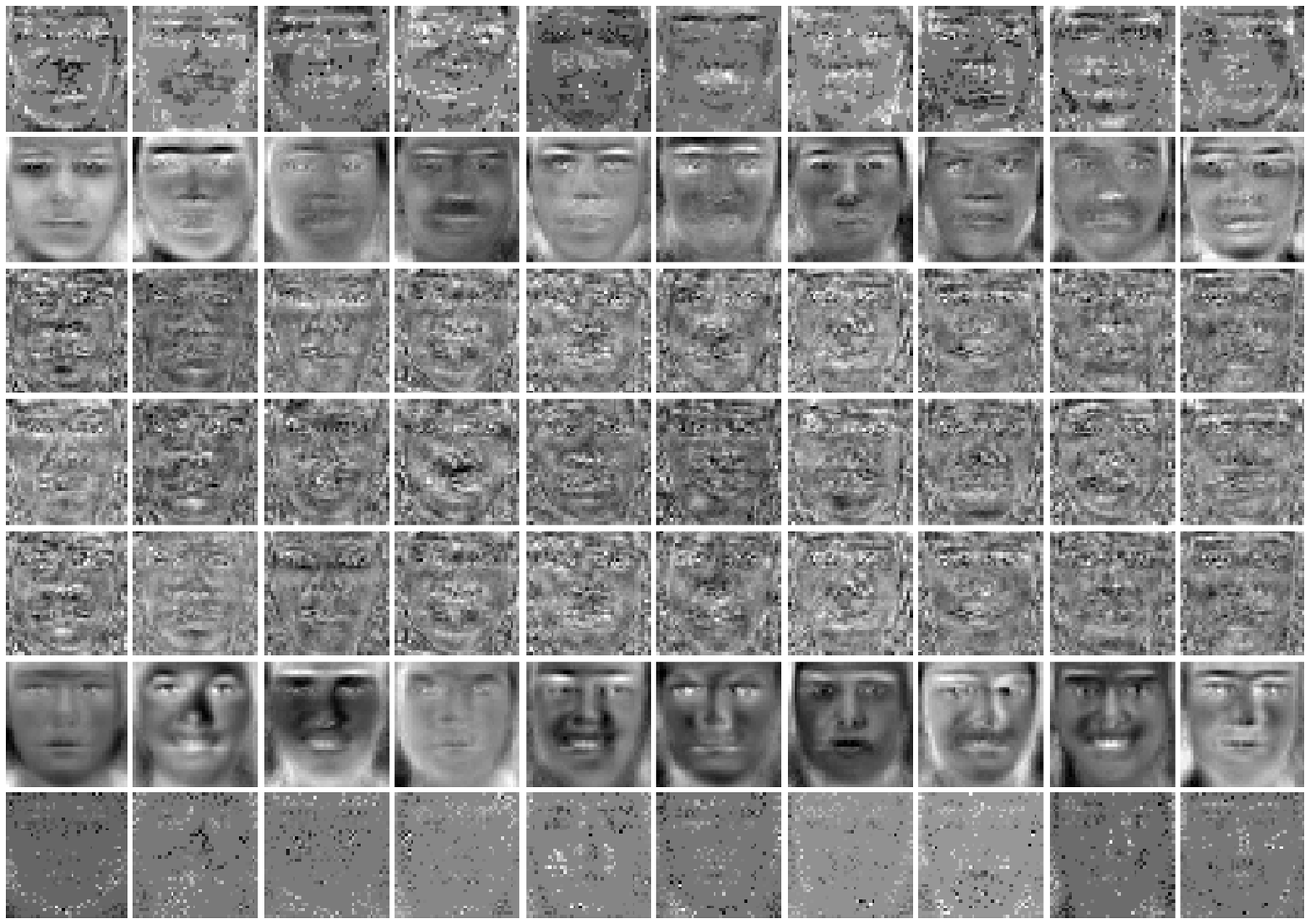}
\end{center}
   \caption{Plots of first $10$ bases obtained from $7$ dimensionality reduction algorithms on FERET
For each column, from top to bottom: MEN, DLA, LPP, NPE, FLDA, PCA, and SPCA}
\label{fig:faceFERET}
\end{figure}

\begin{figure}[htdp]
\begin{center}
 \includegraphics[width=1\linewidth]{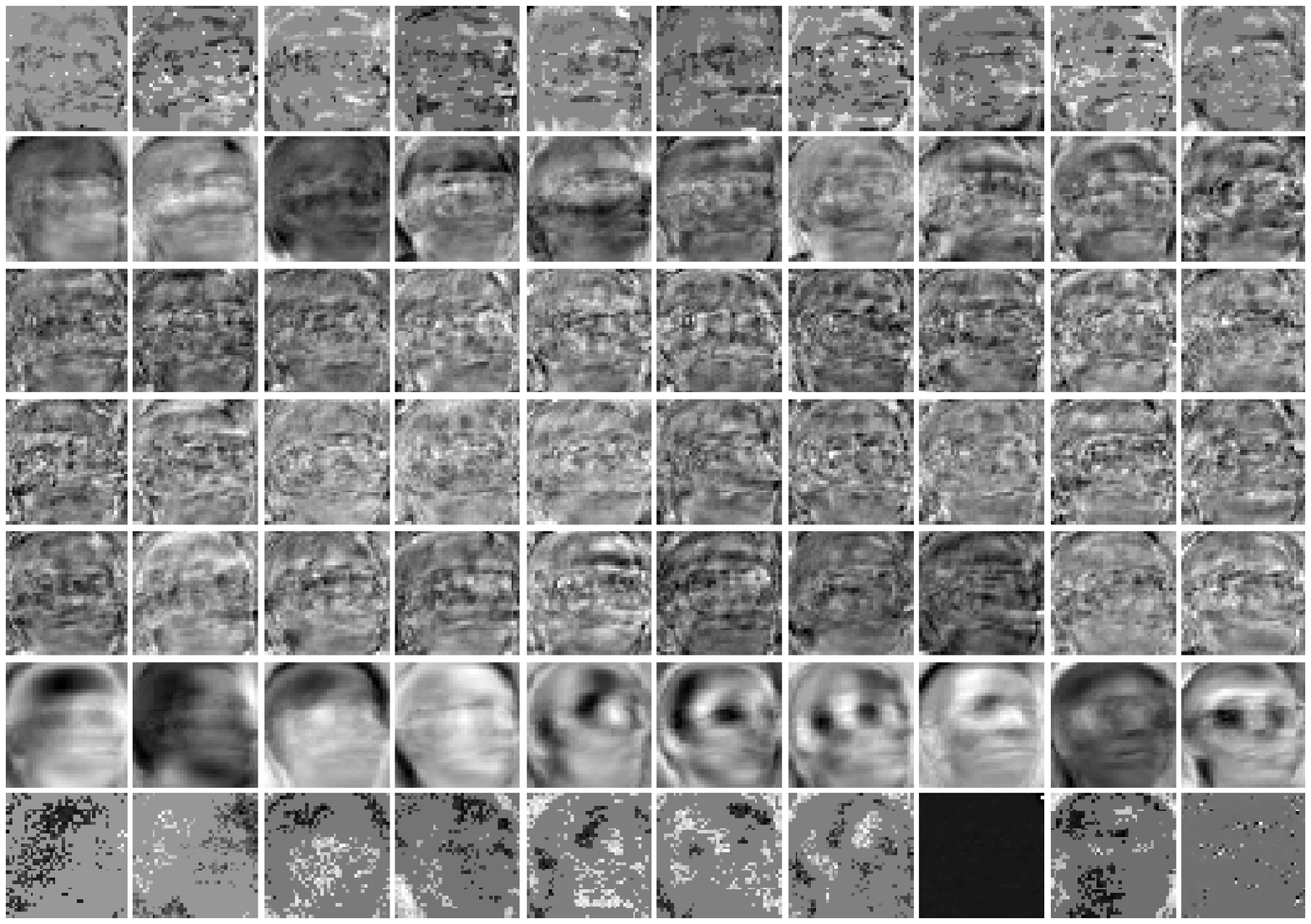}
\end{center}
   \caption{Plots of first $10$ bases obtained from $7$ dimensionality reduction algorithms on UMIST
For each column, from top to bottom: MEN, DLA, LPP, NPE, FLDA, PCA, and SPCA}
\label{fig:faceUMIST}
\end{figure}

\begin{figure}[htdp]
\begin{center}
 \includegraphics[width=1\linewidth]{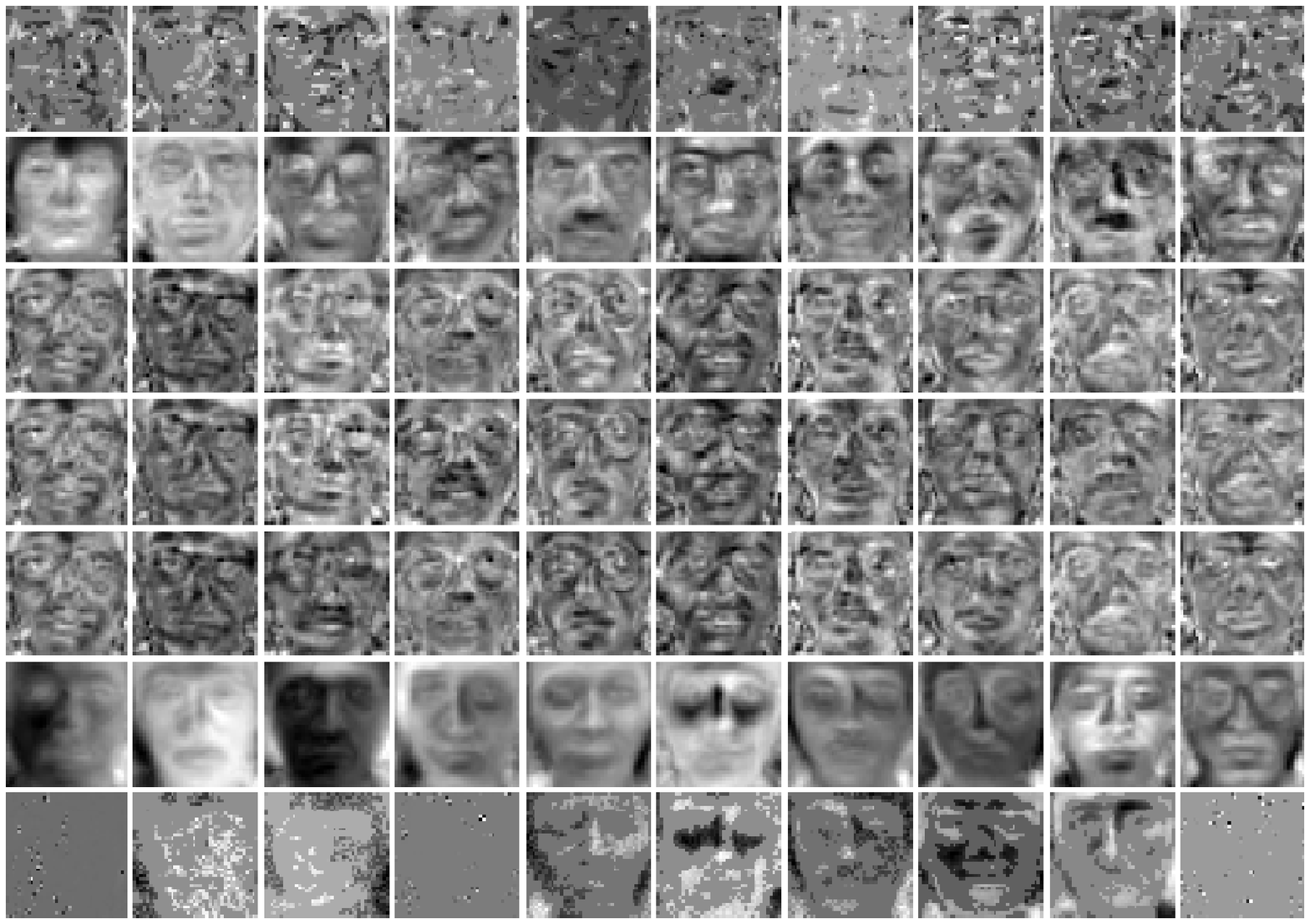}
\end{center}
   \caption{Plots of first $10$ bases obtained from $7$ dimensionality reduction algorithms on YALE
For each column, from top to bottom: MEN, DLA, LPP, NPE, FLDA, PCA, and SPCA}
\label{fig:faceYALE}
\end{figure}

In each LARS loop of the MEN algorithm, according to the algorithm listed in Algorithm 1, all entries of one column in the MEN projection matrix are zeros initially. They are sequentially added into the active set according to their importance. The values of active ones are increased with equal altering correlation. In this process, the $\ell_1$-norm of the column vector is augmented. Fig.\ref{fig:tree} shows the altering tracks of some entries of the column vector in one LARS loop. We called these tracks ``coefficient path'' in LARS. In Fig.\ref{fig:tree}, every coefficient path starts from zero when the corresponding variable becomes active, and changes its direction when another variable is added into the active set. All the paths keep in the directions which make the correlations of their corresponding variables equally altering. The $\ell_1$-norm is increasing along the greedy augment of entries. The coefficient paths proceed along the gradient decent direction of objective function on the subspace, which is spanned by the active variables.

\begin{figure}[htdp]
\begin{center}
 \includegraphics[width=1\linewidth]{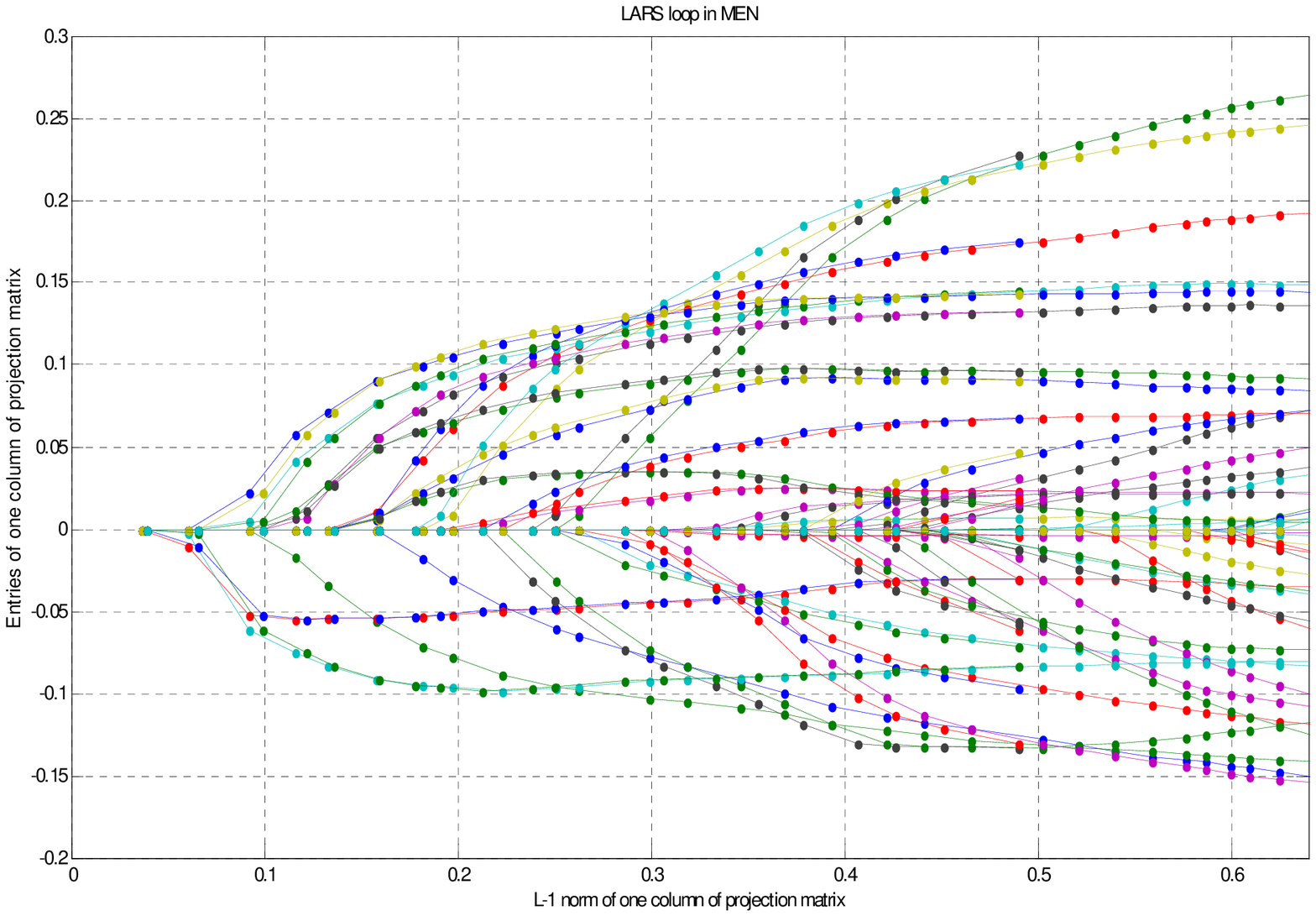}
\end{center}
   \caption{Entries of one column of projection matrix vs. its $\ell_1$-norm in one LARS loop of MEN}
\label{fig:tree}
\end{figure}

Fig.\ref{fig:simulation} shows $10$ of the $1600$ coefficient paths from LAPS loop for the first base in experiment on FERET dataset. MEN selects ten important variables (facial features) sequentially here. Each feature, its corresponding coefficient path and the``MEN fac'' when the feature is added into active set are assigned the same color which is different with the other $9$ features. In each ``MEN face'', the new added active feature is marked by a small circle, and all the active features are marked by white crosses. The features selected by MEN can produce explicit interpretation of the relationship between facial features and face recognition: feature 1 is the left ear, feature 2 is the top of nose, feature 3 is on the head contour, feature 4 is the mouth, feature 5 and feature 6 are on the left eye, feature 7 is the right ear, and feature 8 is the left corner of mouth. These features are already verified of great importance in face recognition by many other famous face recognition methods. Moreover, Fig.\ref{fig:simulation} also shows MEN can group correlated features, for example, feature 5 and feature 6 are selected sequentially because they are both on the left eye. In addition, features which are not very important, such as feature 9 and feature 10 in Fig.\ref{fig:simulation}, are selected after the selection of the other more significant features and assigned smaller value than those more important ones. Therefore, MEN is a powerful algorithm in variable (feature) selection.

\begin{figure}[htdp]
\begin{center}
 \includegraphics[width=1\linewidth]{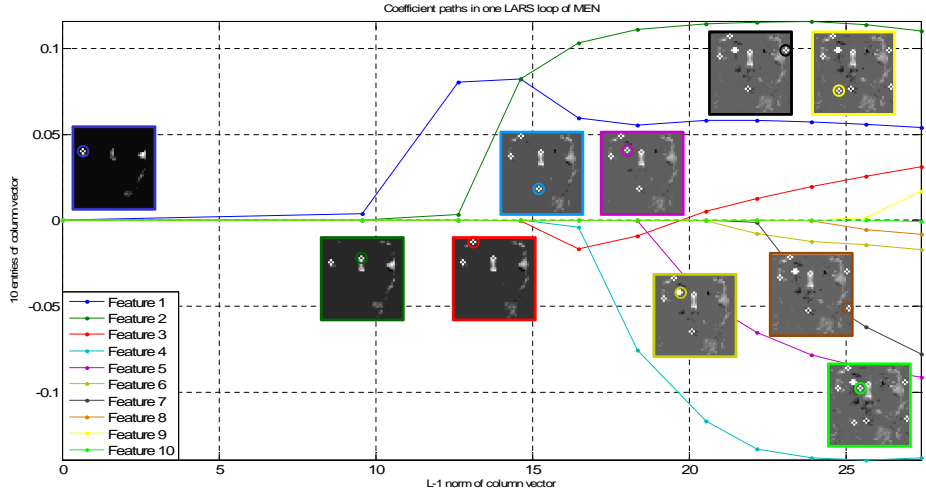}
\end{center}
   \caption{Coefficient paths of $10$ entries (features) in one column vector}
\label{fig:simulation}
\end{figure}


\section{Conclusion}

In this paper, we propose a unifying framework which obtains a sparse projection matrix for subsequent classification, termed manifold elastic net or MEN for short. MEN incorporates the advantages of both manifold learning based dimensionality reduction and sparse learning based dimensionality reduction, but it is not a direct combination of these two. To obtain a sparse projection matrix, MEN imposes the elastic net penalty over a loss function that is defined under the patch alignment framework. The objective function of MEN can be transformed into a lasso penalized least square problem by using a series of complex linear algebra equivalent transformations, and thus the least angle regression (LARS) can be applied to obtain the optimal sparse projection matrix.

In MEN, the patch alignment framework is first used to construct local patches of data and unifies these patches into a global coordinate system. Secondly, the classification error is minimized directly via weighted principal component analysis (PCA) over class centers. Thirdly, to obtain a sparse projection matrix with the grouping effect, the elastic net penalty is added to the objective function. After a series of equivalent transformations, MEN can be rewritten as a lasso-type regression. Therefore, LARS can be applied to solve the problem efficiently. In each LARS loop for MEN optimization, important variables are added into the active set sequentially according to their correlation. All the elements in the active set are altered along a special direction with a special distance in each step. The special direction and distance keep the correlation of active elements identical and the largest in a LARS loop. The procedure is conducted several times to obtain a set of sparse bases because these bases are independent.

MEN enjoys advantages in several aspects: 1) the local geometry of intra-class samples is well preserved for low dimensional data representation, 2) both the margin maximization and the classification error minimization are considered for discriminative information preservation, 3) the sparsity of the projection matrix of MEN improves the parsimony in computation, 4) the elastic net penalty reduces the over-fitting problem, and 5) the projection matrix of MEN can be interpreted psychologically and physiologically.

Experimental results of face recognition on UMIST, FERET and YALE show that MEN performs better and more stable than popular dimensionality reduction algorithms, such as the principal component analysis (PCA), Fisher's linear discriminant analysis (FLDA), the discriminative locality alignment (DLA), the locality preserving projections with supervised setting (LPP), the neighborhood preserving embedding with supervised setting (NPE), and the sparse principal component analysis (SPCA).

There are still many interesting properties of MEN which have not been targeted and formally proved in this paper. In the future, we will analyze its error bounds under different situations. Another important problem in MEN is how to choose the optimal sparsity, so that we can remove most noise and retain most discriminative information for subsequent classification. The compressed sensing may be an effective tool to address the above concern. It is also valuable to replace the lasso penalty with the $\ell_0$-norm penalty to further improve MEN with more ``accurate sparsity''. The lasso penalty is a relaxation of $\ell_0$-norm penalty, and there are alternatives which could perform better than the lasso penalty, e.g., the smoothly clipped absolute deviation penalty (SCAD) (Fan and Li \cite{SCAD}), the reweighted $\ell_1$ minimization (Candes et al. \cite{ReweightedL1}), the adaptive lasso (Zou \cite{AdaptiveLasso}) and the adaptive elastic net (Zou and Zhang \cite{AdaptiveElasticNet}). The advantages of these methods can be adopted in MEN to further enhance the variable selection ability of MEN, and there is still a long way to go.



\section*{Acknowledgement}

This work was supported by NTU NAP Grant with project number M58020010 and the Open Project Program of the State Key Lab of CAD and CG (Grant No. A1006), Zhejiang University.



\end{document}